\documentclass[letterpaper, 10 pt, conference]{ieeeconf}  

\IEEEoverridecommandlockouts                              

\usepackage{graphicx}
\usepackage{graphicx}
\usepackage{amsmath}
\newtheorem{myDef}{Definition}
\newtheorem{myTheo}{Theorem}

\usepackage{bm}
\usepackage{cancel,soul,ulem}
\usepackage{algorithm}
\usepackage{multirow} 
\usepackage{xcolor}
\usepackage{algpseudocode}   
\usepackage{arydshln}
\usepackage{stfloats}
\usepackage{color}
\usepackage{amssymb}
\usepackage{graphicx}

\title{\LARGE \bf
A Right Invariant Extended Kalman Filter for Object based SLAM
}

\author{Yang Song$^{*}$, Zhuqing Zhang$^{\dag}$, Jun Wu$^{\dag}$, Yue Wang$^{\dag}$, Liang Zhao$^{*}$, Shoudong Huang$^{*}$
\thanks{$^{*}$Yang Song, Liang Zhao and Shoudong Huang are with Robotics Institute, University of Technology Sydney, Australia.
      {Email:\tt\small Yang.Song-4@student.uts.edu.au}}%
\thanks{$^{\dag}$Zhuqing Zhang, Jun Wu, and Yue Wang are with the State key Laboratory of Industrial Control and Technology, Zhejiang University, P.R. China.}%
}

\begin{document}

\maketitle
\thispagestyle{empty}
\pagestyle{empty}


\begin{abstract}


With the recent advance of deep learning based object recognition and estimation, it is possible to consider object level SLAM where the pose of each object is estimated in the SLAM process. In this paper, based on a novel Lie group structure, a right invariant extended Kalman filter (RI-EKF) for object based SLAM is proposed. The observability analysis shows that the proposed algorithm automatically maintains the correct unobservable subspace, while standard EKF (Std-EKF) based SLAM algorithm does not. This results in a better consistency for the proposed algorithm comparing to Std-EKF. Finally, simulations and real world experiments validate not only the consistency and accuracy of the proposed algorithm, but also the practicability of the proposed RI-EKF for object based SLAM problem. The MATLAB code of the algorithm is made publicly available\footnote{https://github.com/YangSONG-SLAM/RIEKF$\_$objectSLAM}. 

\end{abstract}


\section{INTRODUCTION}


During the past decade, visual sensors are very popular due to the properties of low cost and rich information, and various kinds of frameworks are designed for visual SLAM \cite{svo,DVO,orbslam2}. However, most of these works only utilize low-level features such as points \cite{svo}, lines \cite{svo-l} or planes \cite{planar slam} and neglect the high-level features such as objects which contain strong geometric constraints \cite{high level}.
High-level features have many advantages over low-level features, including broader perspective loop closures, longer feature tracking and considering the intrinsic constraints among low-level features \cite{high level,semantic slam}. Deep learning boosts the front-end techniques including the detection of high-level features. And the back-end estimation algorithms based on those need to be further investigated.

The early SLAM works considering the object features include   \cite{slam++,mono-obj-slam}, where their efforts are focusing on recognizing object features via traditional methods.  Recent works utilize neural networks to detect the object features and optimization methods to estimate the poses of both the objects and the robot \cite{cubeslam, fusion++,quadricslam}. However, optimization-based methods require much more computation than filter-based methods when the trajectory of robot is very long. Thus they are not suitable for deployment on lightweight platforms. Therefore, developing filter based methods are important. Nevertheless, to the best of our knowledge, works focused on the filter-based SLAM frameworks that consider the object features are still blank. One of the reasons is that the conventional filter-based SLAM (like standard EKF) usually suffers from the problem of inconsistency. And a consistent filter-based estimator requires elaborate modelling and algorithmic design.   

An inconsistent filter-based SLAM system will underestimate the uncertainty of the estimated state, which gradually leads to poor results and even makes the algorithm diverge. The discovery of inconsistency of point feature based EKF SLAM can date back to 2001 \cite{i1}. In the last decades, many works focus on analyzing the cause of inconsistency and proposing improvement algorithms to alleviate the inconsistency of the system \cite{i1,i3,i7,i8,i10}.  According to \cite{i8}, the explanation of such phenomena is that EKF linearizes the system model at the latest estimated values so that Jacobians of the process and observation functions are not estimated at the same state. Furthermore, \cite{i10} analyzes the observability of the system and argues that the ever-changing linearization points break the unobservable subspace of the system, therefore spurious information along the unobservable direction is introduced to the system, which makes the estimate inconsistent. To solve this problem, first Jacobian estimate (FEJ) \cite{i11} and observability constraint (OC) \cite{oc} methods are proposed.  

On the contrary, instead of adding artificial constraint into the estimator, algorithms designed by Lie group theory are found to have the potential to naturally handle invariance including observability constraints in point feature based SLAM algorithms \cite{a10,i14,i15}. Exploiting the properties of invariant tangent vector field, Lie group theory generates the invariant-EKF methodology 
\cite{c1,c4,c6}, which is firstly applied to EKF-SLAM in \cite{c6}. Recently, EKF designed on Lie groups has become popular in filter based SLAM  
\cite{i15,i16,i20,i21}, as many such algorithms perform well in terms of consistency and convergence. 
Based on a specific Lie group representation, \cite{i14} proposed right invariant EKF (RI-EKF) algorithm for 2D point feature based SLAM to alleviate the inconsistency and make the state estimates obtained more accurate. Furthermore, the convergence and consistency for RI-EKF point feature based SLAM in 3D environment are analyzed in \cite{a10}, showing the advantages of the invariant algorithms for SLAM problems. 

However, all filter-based SLAM methods mentioned above only consider the  point features, and how to fuse the object features (poses) consistently is still an untouched problem. In this paper, we propose a consistent EKF algorithm for object based SLAM. To be specific, the contributions of this paper are shown as follows:
\begin{itemize}
    \item An invariant EKF is designed on a new Lie group for SLAM with object features.
    \item The observability analysis of our proposed EKF SLAM with object features, showing it can naturally maintain the correct unobservable subspace.
    \item  The effectiveness of our proposed algorithm is validated  via simulations and real-data experiments.
\end{itemize}

\noindent\textbf{Notations:} In this paper, bold lower-case and upper-case letters are reserved for vectors and matrices/elements in the Lie group, respectively. The notation $\mathbb {SO}(3)$ represents 3D special rotation group, consisting of all rotation transformations in $\mathbb{R}^3$. $\mathfrak {so}(3)$ is the Lie algebra of $\mathbb {SO}(3)$, containing all $3\times 3$ skew symmetric matrices. $(\cdot)^{\land}$ represents the skew symmetric operator that transforms a 3-dimensional vector into a skew symmetric matrix. $\exp^G$ represents the exponential map on a Lie group $G$. $\log^G$ is the inverse of exponential map on a Lie group $G$. $N(\textbf{0},\textbf{P})$ represents a zero mean Gaussian distribution with covariance $\textbf{P}$.

\section{Object based SLAM Problem}
An object feature considered in this work is represented as a 3D pose of the object in the environment. A robot moves in an unknown 3D environment and observes some object features. The object based SLAM focuses on estimating the current robot pose and the poses of all the object features using the process model and observation model.


\subsubsection{State Space}
An object feature is defined as
\begin{equation}
	(\textbf{R}^f,\textbf{p}^f),
\end{equation}
where $\textbf{R}^f \in \mathbb {SO}(3)$ and $\textbf{p}^f \in \mathbb{R}^3$ are the rotation and position of feature, respectively. 


The set consists of all states combining with the robot pose and $K$ observed features is denoted as $\mathcal G_K$, where 
\begin{equation}
	\begin{aligned}
		\mathcal G_K=\left\{(\textbf{R}^r,\textbf{R}^{f_1},\cdots,\textbf{R}^{f_K},\textbf{p}^r,\textbf{p}^{f_1},\cdots,\textbf{p}^{f_K})\right.\\|\textbf{R}^r, \textbf{R}^{f_j}\in \mathbb{SO}(3),\ \textbf{p}^r, \textbf{p}^{f_j}\in \mathbb{R}^3 \},
	\end{aligned}
\end{equation}
$\textbf{R}^r,\textbf{R}^{f_j} \in \mathbb{SO}(3)$ represent the rotations of robot and the $j$-th feature respectively, and $\textbf{p}^r,\textbf{p}^{f_j} \in \mathbb{R}^3$ represent the position of robot and the $j$-th feature respectively, all described in the global coordinate system\footnote{Sometimes we use the notation $\mathcal G$ instead of $\mathcal G_K$ for brevity. Also, in the remaining of this paper, without losing generality, we assume that there is only one object feature, i.e. $K=1$, to simplify the equations.}.

\subsubsection{Process Model and Observation Model}

Since $ \mathbb{R}^3 \cong \mathfrak {so}(3)$, for simplification, we can define the exponential map of $\mathbb {SO}(3)$ as follows: For $\bm{\xi} \in \mathbb{R}^3$,
\begin{equation}
	\begin{array}[c]{cccc}
		\exp^{\mathbb {SO}(3)}: &\mathbb{R}^3& \rightarrow &\mathbb {SO}(3)\\
		&\bm{\xi}&\rightarrow&\sum_{k=0}^{\infty}\frac{(\bm{\xi}^{\land})^k}{k!}.
	\end{array}
\end{equation}

 The following first-order integration scheme of discrete noisy process model is widely used \cite{i11,i14}: 
\begin{equation}
	\begin{array}{lll}
		\textbf{X}_{n+1}=f(\textbf{X}_n,\textbf{U}_n,\textbf{w}_n)\\
		=(\textbf{R}^r_{n}\exp^{\mathbb{SO}(3)}(\textbf{w}_n^R)\textbf{R}^u_n,\textbf{R}^{f}_{n},\textbf{p}^r_{n}+\textbf{R}^r_{n}(\textbf{p}^u+\textbf{w}_n^p),\textbf{p}^f_{n}),
	\end{array}
	\label{ProcessModel}
\end{equation}
where $\textbf{X}_{i}=(\textbf{R}^r_{i},\textbf{R}_{i}^{f},\textbf{p}^r_{i},\textbf{p}^f_{i})$ is the state at time step $i,\ i=n,n+1$, $\textbf{U}_n=(\textbf{R}^u_n,\textbf{p}^u_n)$ is the odometry,  $\textbf{w}_n=((\textbf{w}^R_n)^T,(\textbf{w}^p_n)^T)^T(\in \mathbb{R}^6)\sim  N(\textbf{0},\mathbf{\Sigma}_n)$ is the odometry noise.



After the object detection and matching from the SLAM front-end, the observation can be regarded as relative poses of object features in the current robot frame. Then the observation model for the object feature, $(\textbf{R}^f,\textbf{p}^f)$, in the $n+1$-th robot frame, $(\textbf{R}^r_{n+1},\textbf{p}^r_{n+1})$, is constructed as
\begin{equation}
		\textbf{Z}=h(\textbf{X}_{n+1},\textbf{v}_{n+1})=(\textbf{R}^z,\textbf{p}^z),
		\label{ObsModel}
\end{equation}
where 
$$\begin{array}{rl}
	\textbf{R}^z=&\exp^{\mathbb{SO}(3)}(\textbf{v}^R_{n+1})(\textbf{R}^r_{n+1})^T\textbf{R}_{n+1}^f,\\
	\textbf{p}^z=&(\textbf{R}^r_{n+1})^T(\textbf{p}^f_{n+1}-\textbf{p}^r_{n+1})+\textbf{v}^p_{n+1},
\end{array}$$ and
$\textbf{v}_{n+1}=((\textbf{v}^R_{n+1})^T, (\textbf{v}^p_{n+1})^T)^T(\in \mathbb{R}^6) \sim N(\textbf{0},\bm{\Omega_{n+1}})$ is the observation noise.

\section{RI-EKF for Object based SLAM}
\subsection{RI-EKF framework for object based SLAM}
\subsubsection{A novel Lie group structure on state space}
For all $(\textbf{R}_i,\textbf{R}^{f}_i,\textbf{p}^r_i,\textbf{p}^{f}_i) \in \mathcal G$, $i=1,2$, an operator $\oplus$ is defined by
\begin{equation}
	\begin{aligned}
		(\textbf{R}^r_1,\textbf{R}^{f}_1,\textbf{p}^r_1,\textbf{p}^{f}_1) \oplus (\textbf{R}^r_2,\textbf{R}^{f}_2,\textbf{p}^r_2,\textbf{p}^{f}_2)=\\(\textbf{R}^r_1\textbf{R}^r_2,\textbf{R}^{f}_1\textbf{R}^{f}_2,\textbf{R}^r_1\textbf{p}^r_2+\textbf{p}^r_1,\textbf{R}^r_1\textbf{p}^{f}_2+\textbf{p}^{f}_1).
	\end{aligned}
\label{Lie_action}
\end{equation}
Then we can check that equiped with $\oplus$ defined in (\ref{Lie_action}), the state space $\mathcal G$ becomes a Lie group and is isomotric to $\mathbb{SE}_{K+1}(3) \times (\mathbb{SO}(3))^K$, where $K$ (set as 1 for simplicity) is the number of observed  features. The Lie group $\mathbb{SE}_{K+1}(3)$, defined in \cite{a10,i14}, plays a significant role in RI-EKF for point feature based SLAM. The notation $\ominus$, the minus of $\oplus$, is defined by $\textbf{X}_a\ominus \textbf{X}_b=\textbf{X}_a\oplus \textbf{X}_b^{-1}$ and $\textbf{X}_b\oplus\textbf{X}_b^{-1}=\textbf{X}_b^{-1}\oplus\textbf{X}_b=(\textbf{I},\textbf{I},\textbf{0},\textbf{0})$. 
Denote $\mathfrak{g}$ as the Lie algebra of $\mathcal G$. And we have  $\mathfrak{g}\cong se_{K+1}(3)\times (so(3))^K \cong \mathbb{R}^{6+6K}$. 
Therefore,  the form of an element $\bm{\xi}$ in Lie algebra $\mathfrak{g}$ can be constructed as
\begin{equation}\label{eq:nonlinear_error}
	\bm{\xi}^T=((\bm{\xi}^{R^r})^T,(\bm{\xi}^{R^f})^T,(\bm{\xi}^{p^r})^T,(\bm{\xi}^{p^f})^T),
\end{equation}
where $\bm{\xi}^{R^r},\bm{\xi}^{R^f},\bm{\xi}^{p^r},\bm{\xi}^{p^f} \in \mathbb{R}^3$.
The exponential map $\exp^{\mathcal G}$ on this Lie group  can  be defined by
\begin{equation}
    \begin{aligned}
		\exp^{\mathcal G}(\bm{\xi})=&(\exp^{\mathbb {SO}(3)}(\bm{\xi}^{R^r}),\exp^{\mathbb {SO}(3)}(\bm{\xi}^{R^f})\\& J_l(\bm{\xi}^{R^r})\bm{\xi}^{p^r},J_l(\bm{\xi}^{R^r})\bm{\xi}^f),
	\end{aligned}
\end{equation}
where \begin{equation}
	J_l(\bm{\xi}^{R^r})=\sum_{k=0}^{\infty}\frac{((\bm{\xi}^{R^r})^{\land})^k}{(k+1)!}.
\end{equation}
Then an error state $\bm{\xi}$ for an estimated state $\hat{\textbf{X}}$ satisfies
\begin{equation}
    \textbf{X}=\exp^{\mathcal G}(\bm{\xi})\oplus\hat{\textbf{X}},
\end{equation}
where $\textbf{X}$ represents the true state.

\subsubsection{Propagation}
Based on the Lie group structure introduced above, the process model (\ref{ProcessModel}) becomes 
\begin{equation}
	\begin{aligned}
		\textbf{X}_{n+1}=\textbf{X}_n\oplus (\exp^{\mathbb{SO}(3)}(\textbf{w}_n^R)\textbf{R}^u_n,\textbf{I}_3,\textbf{p}^u_n+\textbf{w}^p_n,\textbf{0}_{3\times1}).
	\end{aligned}	
\end{equation}
The predicted state, $\textbf{X}_{n+1|n}$, by propagation is computed by
\begin{equation}
	\textbf{X}_{n+1|n}=(\textbf{R}^r_{n|n}\textbf{R}^u_n,\textbf{R}^{f}_{n|n},\textbf{R}^r_{n|n}\textbf{p}^u_n+\textbf{p}^r_{n|n},\textbf{p}^f_{n|n}).
\end{equation}
where $\textbf{X}_{n|n}=(\textbf{R}^r_{n|n},\textbf{R}^{f}_{n|n},\textbf{p}^r_{n|n},\textbf{p}^f_{n|n})$ is the updated state at time $n$. 
And the estimated error $\bm{\xi}_{n+1|n}$ by propagation is 
\begin{equation}
	\begin{array}{lll}
		\bm{\xi}_{n+1|n}&\dot{=}\log(\textbf{X}_{n+1}\ominus \textbf{X}_{n+1|n})\\
		&\approx \textbf{F}_n\bm{\xi}_{n|n}+\textbf{G}_n{\textbf{w}}_n,
	\end{array}
\end{equation}
where $\bm{\xi}_{n|n} \sim N(\textbf{0},\bm{P}_n)$ is the estimation error  for $\textbf{X}_{n|n}$,  ${\textbf{w}}_n=((\textbf{w}^R_n)^T,(\textbf{w}^p_n)^T)^T$ is the odometry noise, the coefficient matrices $\textbf{F}_n$ and $\textbf{G}_n$ in RI-EKF are
\begin{equation}
	\begin{array}{lll}
		\textbf{F}_n=\textbf{I}_{6+6K},\\
		\textbf{G}_n=\left[
		\begin{array}{cccc}
			\textbf{R}^r_{n|n}& \textbf{0}_{3\times 3}\\ \textbf{0}_{3\times 3}& \textbf{0}_{3\times 3}\\
			(\textbf{p}^r_{n|n}+\textbf{R}^r_{n|n}\textbf{p}^u_n)^{\land}\textbf{R}^r_{n|n}&\textbf{R}^r_{n|n}\\
			(\textbf{p}^f_{n|n})^{\land}\textbf{R}^r_{n|n}& \textbf{0}_{3\times 3}
		\end{array}
		\right].\\
	\end{array}
	\label{FG}
\end{equation}
Then, the covariance matrix of the state $\textbf{X}_{n+1|n}$ by propagation is 
\begin{equation}
	\textbf{P}_{n+1|n}={\textbf{{F}}}_{n} {\textbf{P}}_n  {\textbf{{F}}}^{T}_{n} + {\textbf{{G}}}_n{\bm{\Sigma}}_{n} {\textbf{{G}}}^{T}_{n}.
\end{equation}

\subsubsection{Update}
Suppose $\textbf{Z}$ is an observation in (\ref{ObsModel}). By introducing a new minus operator $\boxminus$ for the observation, innovation $y$ is defined as 
\begin{equation}
	\begin{array}{lll}
		&\textbf{y}=\left[
		\begin{array}{ccc}
			\textbf{y}^R\\
			\textbf{y}^p\\
		\end{array}
		\right]=\textbf{Z}\boxminus [(\hat{\textbf{R}}^r)^T\hat{\textbf{R}}^f,(\hat{\textbf{R}}^r)^T(\hat{\textbf{p}}^f-\hat{\textbf{p}}^r)]\\
		&\doteq\left[
		\begin{array}{ccc}
			\log^{\mathbb{SO}(3)}(\exp^{\mathbb{SO}(3)}((\textbf{v}^R)^{\land})(\textbf{R}^r)^T\textbf{R}^f(\hat{\textbf{R}}^f)^T\hat{\textbf{R}}^r)\\ (\textbf{R}^r)^T(\textbf{p}^f-\textbf{p}^r)+\textbf{v}^p-[(\hat{\textbf{R}}^r)^T(\hat{\textbf{p}}^f-\hat{\textbf{p}}^r)]\\
		\end{array}
		\right],
	\end{array}
	\label{Y}
\end{equation}
where $\textbf{y}^R, \textbf{y}^p\in \mathbb{R}^3$.
The linearization of $\textbf{y}^R$ is obtained by
\begin{equation}
	\begin{array}{lll}
		\textbf{I}_3+(\textbf{y}^R)^{\land}&\approx \exp^{\mathbb{SO}(3)}((\textbf{y}^R)^{\land})\\
		&=\exp^{\mathbb{SO}(3)}((\textbf{v}^R)^{\land})(\textbf{R}^r)^T\textbf{R}^f(\hat{\textbf{R}}^f)^T\hat{\textbf{R}}^r\\
		&\approx \textbf{I}_3-((\hat{\textbf{R}}^r)^T\bm{\xi}^{R^r})^{\land}\\&\ \ \ +((\hat{\textbf{R}}^r)^T\bm{\xi}^{R^f})^{\land}+(\textbf{v}^R)^{\land}.
	\end{array}
\end{equation}
Then, by omitting the second-order small quantities, we have
\begin{equation}
	\begin{array}{lll}
		\textbf{y}^R=-(\hat{\textbf{R}}^r)^T\bm{\xi}^{R^r}+(\hat{\textbf{R}}^r)^T\bm{\xi}^{R^f}+\textbf{v}^R.
	\end{array}
\end{equation}
The linearization of $\textbf{y}^p$ can be derived directly from point feature RI-EKF SLAM in \cite{a10,i14}. For the innovation at time step $(n+1)$, we have
\begin{equation}
	\begin{array}{lll}
		\textbf{y}_{n+1}=\textbf{H}_{n+1}\bm{\xi}_{n+1|n}+\textbf{v}_{n+1},\\
	\end{array}
	\label{H}
\end{equation}
where 
\begin{equation}
	\textbf{H}_{n+1}=\left[
	\begin{array}{cccc}
		\textbf{H}^{R,R^r}_{n+1}& \textbf{H}^{R,R^f}_{n+1}& \textbf{0}_{3\times 3}& \textbf{0}_{3\times 3}\\
		\textbf{0}_{3\times 3}& \textbf{0}_{3\times 3}& \textbf{H}^{p,p^r}_{n+1}& \textbf{H}^{p,p^f}_{n+1}\\
	\end{array}
	\right],
\end{equation}
$$\begin{aligned}
	&\textbf{H}^{R,R^r}_{n+1}=\textbf{H}^{p,p^r}_{n+1}=-({\textbf{R}}_{n+1|n}^r)^T,\\
	&\textbf{H}^{R,R^f}_{n+1}=\textbf{H}^{p,p^f}_{n+1}=({\textbf{R}}_{n+1|n}^r)^T,
\end{aligned}$$ and
$\textbf{v}_{n+1}\sim N(\textbf{0},\bm{\Omega}_{n+1})$ is the observation noise.
Then the state is updated by
\begin{equation}
	\textbf{X}_{n+1|n+1}=\exp^{\mathcal{G}}(\bm{\xi}_{n+1|n+1})\oplus \textbf{X}_{n+1|n},
\end{equation}
where $\bm{\xi}_{n+1|n+1}={\textbf{K}}_{n+1}{\textbf{y}}_{n+1}$ is the update state error vector, and
$${\textbf{K}}_{n+1}={\textbf{P}}_{n+1|n} {{\textbf{H}}}_{n+1}^{T} ({{\textbf{H}}}_{n+1} {\textbf{P}}_{n+1|n}{{\textbf{H}}}_{n+1}^{T}+{\bm{\Omega}}_{n+1})^{-1}.$$
Its covariance is updated as
\begin{equation}
	{{\textbf{P}}}_{n+1} = ({\textbf{I}}-{{\textbf{K}}}_{n+1}{{\textbf{H}}}_{n+1}){\textbf{P}}_{n+1|n}.
\end{equation}

The whole process of RI-EKF integrating object features can be summarized  in \textbf{Algorithm \ref{Alg:RI-EKF}}.

\begin{algorithm}[H]
	$\textbf{F}_n$, $\textbf{G}_n$, $\textbf{H}_{n+1}$ are given in (\ref{FG}) and (\ref{H}).\\
	\hspace*{0.02in} {\bf Input:} $\textbf{X}_{n|n}$, ${\textbf{P}}_n$, ${\textbf{U}}_n$, ${\textbf{Z}}_{n+1}$\\
	\hspace*{0.02in} {\bf Output:} $\textbf{X}_{n+1|n+1}$, ${\textbf{P}}_{n+1}$\\
	\hspace*{0.02in} {\bf Propagation:} \\
	$ \textbf{X}_{n+1|n}  \leftarrow   \textbf{X}_{n|n}\oplus \textbf{U}_n$ \\
	${\textbf{P}}_{n+1|n} \leftarrow {{\textbf{F}}}_{n} {\textbf{P}}_n  {{\textbf{F}}}^{T}_{n} + {{\textbf{G}}}_n \bm{\Sigma}_{n} {{\textbf{G}}}^{T}_{n} $\\
	\hspace*{0.02in} {\bf Update:} \\	
	${{\textbf{K}}}_{n+1} \leftarrow {\textbf{P}}_{n+1|n} {{\textbf{H}}}_{n+1}^{T} ({{\textbf{H}}}_{n+1} {\textbf{P}}_{n+1|n}{\textbf{{H}}}_{n+1}^{T}+\bm{\Omega}_{n+1}  )^{-1} $\\ 
	${\textbf{y}}_{n+1} \leftarrow {\textbf{Z}}_{n+1}\boxminus h_{n+1}(\textbf{X}_{n+1|n}, {\textbf{0}} )$ \\
	$\textbf{X}_{n+1|n+1} \leftarrow  \exp^{\mathcal{G}}({{\textbf{K}}}_{n+1}{\textbf{y}}_{n+1})\oplus \textbf{X}_{n+1|n}$\\ 
	${{\textbf{P}}}_{n+1} \leftarrow ({\textbf{I}}-{{\textbf{K}}}_{n+1}{{\textbf{H}}}_{n+1}){\textbf{P}}_{n+1|n}$\\
	\caption{RI-EKF for Object based SLAM}
	\label{Alg:RI-EKF}
\end{algorithm}

\subsection{New Feature Initialization}\label{sec:initialization}
Besides propagation and updating, another indispensable procedure is object feature initialization. This subsection will propose the method to augment the estimated state $\textbf{X} \in \mathcal{G}$ and the covariance matrix $\textbf{P}$ when the robot observes a new object feature with the observation $\textbf{Z}=(\textbf{R}^z,\textbf{p}^z) \in \mathbb{SO}(3)\times \mathbb{R}^3$.

Suppose the state error $\bm{\xi} \in \mathfrak{g}\cong \mathbb{R}^{6K+6}$ is of the form
\begin{equation}
	\begin{array}{lll}
		\bm{\xi}&=\left[
		\begin{array}{llll}
			\bm{\xi}^{R}\\
			\bm{\xi}^{p}\end{array} \right]
		=\log^{\mathcal{G}}(\textbf{X}\ominus \hat{\textbf{X}}),
	\end{array}
\label{xi_form}
\end{equation}
where $$\bm{\xi}^{R}=((\bm{\xi}^{R^r})^T,(\bm{\xi}^{R^{f}})^T)^T$$ is the rotation part, and $$\bm{\xi}^{p}=((\bm{\xi}^{p^r})^T,(\bm{\xi}^{p^{f}})^T)^T$$
is the position part.

Suppose the error $\bm{\xi}\sim N(\textbf{0},\textbf{P})$. Then by the mathematical derivation in Appendix A, 
the expectation of the new object feature, $(\hat{\textbf{R}}^{f_{\text{new}}},\hat{\textbf{p}}^{f_{\text{new}}})$, related to the observation $\textbf{Z}$ is given by
\begin{equation}
	\begin{array}{rll}
		\hat{\textbf{R}}^{f_{\text{new}}}&=\hat{\textbf{R}}^r\textbf{R}^z,\\
		\hat{\textbf{p}}^{f_{\text{new}}}&=\hat{\textbf{p}}^r+\hat{\textbf{R}}^r\textbf{p}^z.
	\end{array}\label{expectations}
\end{equation}
And the augmented covariance $\textbf{P}_{\text{aug}}$ is computed as
\begin{equation}
	\textbf{P}_{\text{aug}}=\left[
	\begin{array}{cccccc}
		\textbf{P}^{R,R}& \textbf{P}^{R,R}\textbf{M}_1^{T}& \textbf{P}^{R,p}& \textbf{P}^{R,p}\textbf{M}_2^{T}\\
		\textbf{M}_1P^{R,R}& \textbf{P}_f^{R,R}& \textbf{M}_1\textbf{P}^{R,p}& \textbf{P}_f^{R,p}\\
		\textbf{P}^{p,R}& \textbf{P}^{p,R}\textbf{M}_1^{T}& \textbf{P}^{p,p}& \textbf{P}^{p,p}\textbf{M}_2^{T}\\
		\textbf{M}_2\textbf{P}^{p,R}& (\textbf{P}_f^{R,p})^T& \textbf{M}_2\textbf{P}^{p,p}& \textbf{P}_f^{p,p}\\
	\end{array}
	\right],
	\label{Pnew}
\end{equation}
where
\begin{equation}
	\begin{array}{rll}
		\textbf{M}_1&=[\textbf{I}_3 \ \textbf{0}_{3,3K}],\\
		\textbf{M}_2&=[\textbf{I}_3 \ \textbf{0}_{3,3K}],\\
		\textbf{P}_f^{R,R}&=\textbf{M}_1\textbf{P}^{R,R}\textbf{M}_1^{T}+\hat{\textbf{R}}^r\bm{\Omega}^{R,R}(\hat{\textbf{R}}^r)^T,\\
		\textbf{P}_f^{R,p}&=\textbf{M}_1\textbf{P}^{R,p}\textbf{M}_2^{T}+\hat{\textbf{R}}^r\bm{\Omega}^{R,p}(\hat{\textbf{R}}^r)^T,\\
		\textbf{P}_f^{p,p}&=\textbf{M}_2\textbf{P}^{p,p}\textbf{M}_2^{T}+\hat{\textbf{R}}^r\bm{\Omega}^{p,p}(\hat{\textbf{R}}^r)^T.
	\end{array}
\end{equation}
The whole process to augment the state is summarized in \textbf{Algorithm \ref{Alg:InitialMethod}}.
\begin{algorithm}[H]
	\hspace*{0.02in} {\bf Input:} \\
	The state and its covariance before augmentation: \\
	$\hat{\textbf{X}}=\left[
	\begin{array}{cccccccc}
		\hat{\textbf{R}}^r& \hat{\textbf{R}}^{f}&\hat{\textbf{p}}^r& \hat{\textbf{p}}^f\\
	\end{array}
	\right]$\\
	${\textbf{P}}=\left[
	\begin{array}{ccc}
		\textbf{P}^{R,R}& \textbf{P}^{R,p}\\
		\textbf{P}^{p,R}& \textbf{P}^{p,p}
	\end{array}
	\right]$\\
	The observation of new feature: $\textbf{Z}=(\textbf{R}^z,\textbf{p}^z)\in \mathbb{SO}(3)\times \mathbb{R}^3$\\
	The covariance of observation noise: $\bm{\Omega}=\left[
	\begin{array}{ccc}
		\bm{\Omega}^{R,R}& \bm{\Omega}^{R,p}\\
		\bm{\Omega}^{p,R}& \bm{\Omega}^{p,p}
	\end{array}
	\right]$\\ \\
	\hspace*{0.02in} {\bf Output:} 
	The augmented state and its covariance:\\
	$\hat{\textbf{X}}_{\text{aug}}=\left[
	\begin{array}{cccccccccc}
		\hat{\textbf{R}}^r& \hat{\textbf{R}}^{f}& \hat{\textbf{R}}^r\textbf{R}^z& \hat{\textbf{p}}^r& \hat{\textbf{p}}^f& \hat{\textbf{p}}^r+\hat{\textbf{R}}^r\textbf{p}^z\\
	\end{array}
	\right]$\\ \\
	${\textbf{P}}_{\text{aug}}
	$ is obtained by (\ref{Pnew})
	\caption{New Feature Initialization}
	\label{Alg:InitialMethod}
\end{algorithm}

\section{Observability Analysis}\label{Ob_anal}
Based on previous research about inconsistency \cite{i7,i8,i10,i11,oc,i14}, the EKF SLAM inconsistency is mainly caused by the violation of the observability constraints. A consistent EKF SLAM estimator should satisfy the following observability constraints: the unobservable subspace for the system model of the estimator is the same as that of the real system (the ideal case where the Jacobians are evaluated at the true state). In this section, we prove that our RI-EKF for object based SLAM can automatically maintain the observability constraints. On the contrary, standard EKF for object based SLAM briefly introduced is unable to maintain the observability constraints. These  explain the better performance by our algorithms in terms of consistency in the following experiments. 
\begin{myDef} 
	The unobservable subspace $\mathcal {\hat{N}}$ 
	based on the state estimates
	is the null space of the corresponding observability matrix   $\mathcal{\hat{\mathbf{O}}}$, 
	where \begin{equation}
		\mathcal{\hat{\mathbf{O}}}=\left[
		\begin{array}{cccccc}
			\hat{\textbf{H}}_0\\
			\hat{\textbf{H}}_1\hat{\textbf{F}}_{0,0}\\
			\vdots\\
			\hat{\textbf{H}}_{n+1}\hat{\textbf{F}}_{n,0}
		\end{array}
		\right],
	\end{equation}
	$\hat{\textbf{H}}_{i}$ is the Jacobian matrix for the $i$-th step observation model evaluated at the state estimate $\hat{\textbf{X}}_{i}$, and $\hat{\textbf{F}}_{i,0}=\hat{\textbf{F}}_i\hat{\textbf{F}}_{i-1}\cdots \tilde{\textbf{F}}_0$, $\hat{\textbf{F}}_j$  is the Jacobian matrix for the $j$-th step propagation model of the estimator evaluated at the state $\hat{\textbf{X}}_j$, $j=0,\cdots,i$. 
	If the models are linearized at the  ground truth,  
	The unobservable subspace based on the true states is denoted by $\mathcal {\breve{N}}$, and the corresponding observability matrix is denoted by $\mathcal{\breve{\mathbf{O}}}$.
\end{myDef}

\subsection{Observability Analysis for RI-EKF}
\begin{myTheo}
	For RI-EKF, the unobservable subspace $\mathcal{\hat{N}}^{RI}$   is the same as $\mathcal{\breve{N}}^{RI}$,  where
	\begin{equation}
		\mathcal{\hat{N}}^{RI}=\mathcal{\breve{N}}^{RI}=\mathop{\text{span}} _{col.}\left[
		\begin{array}{cccccc}
			\textbf{I}_{3}& \textbf{0}_{3,3}\\
			\textbf{I}_{3}& \textbf{0}_{3,3}\\
			\textbf{0}_{3,3}& \textbf{I}_{3}\\
			\textbf{0}_{3,3}& \textbf{I}_{3}
		\end{array}
		\right],
	\end{equation}
	and $\dim(\mathcal{\hat{N}}^{RI})=\dim(\mathcal{\breve{N}}^{RI})=6$.
	\label{RI_OC}
\end{myTheo}

\begin{proof}
    See Appendix B. 
\end{proof} 

Therefore, Theorem \ref{RI_OC} shows that RI-EKF automatically maintains the correct unobservable subspace, which will significantly improve the consistency.

\subsection{Standard EKF for Object based SLAM}
The standard EKF (Std-EKF) for object based SLAM is $\mathbb{SO}(3)$-EKF, whose state space is isomorphic to $(\mathbb{SO}(3))^{K+1}\times (\mathbb{R}^3)^{K+1}$. Suppose there is only one feature in the state vector, an error state $\bm{\eta}\in (\mathfrak{so}(3))^{K+1}\times(\mathbb{R}^3)^{K+1}$, $K=1$, in standard EKF is obtained by
\begin{equation}
    \begin{array}{ll}
    \bm{\eta}&=(\bm{\eta}^{R^r},\bm{\eta}^{R^f},\bm{\eta}^{p^r},\bm{\eta}^{p^f})\\
    &=(log^{\mathbb{SO}(3)}(\textbf{R}^r(\hat{\textbf{R}}^r)^T),log^{\mathbb{SO}(3)}(\textbf{R}^f(\hat{\textbf{R}}^f)^T),\\
    &\ \ \ \ \ \textbf{p}^r-\hat{\textbf{p}}^r,\textbf{p}^f-\hat{\textbf{p}}^f),
    \end{array}
\end{equation}
where $(\textbf{R}^r,\textbf{p}^r)$ and $(\hat{\textbf{R}}^r,\hat{\textbf{p}}^r)$ are the true and the estimated robot poses, $(\textbf{R}^f,\textbf{p}^f)$ and $(\hat{\textbf{R}}^f,\hat{\textbf{p}}^f)$ are the true and the estimated object features, respectively.
Based on this linearization method above, the Jacobians of considered system  are 
\begin{equation}
	\begin{array}{lll}
		\textbf{F}^{Std}_n=\left[
		\begin{array}{cccc}
			\textbf{I}_3& \textbf{0}_{3\times 3}& \textbf{0}_{3\times 3}& \textbf{0}_{3\times 3}\\
			\textbf{0}_{3\times 3}& \textbf{I}_3& \textbf{0}_{3\times 3}&  \textbf{0}_{3\times 3}\\
			-(\textbf{R}^r_{n|n}\textbf{p}^{u})^{\land}& \textbf{0}_{3\times 3}& \textbf{I}_3& \textbf{0}_{3\times 3}\\
			\textbf{0}_{3\times 3}& \textbf{0}_{3\times 3}& \textbf{0}_{3\times 3}& \textbf{I}_3
		\end{array}
		\right],\\
		\textbf{G}^{Std}_n=\left[
		\begin{array}{cccc}
			\textbf{R}^r_{n|n}& \textbf{0}_{3\times 3}& \textbf{0}_{3\times 3}& \textbf{0}_{3\times 3}\\
			\textbf{0}_{3\times 3}& \textbf{0}_{3\times 3}& \textbf{0}_{3\times 3}&  \textbf{0}_{3\times 3}\\
			\textbf{0}_{3\times 3}& \textbf{0}_{3\times 3}& \textbf{R}^r_{n|n}& \textbf{0}_{3\times 3}\\
			\textbf{0}_{3\times 3}& \textbf{0}_{3\times 3}& \textbf{0}_{3\times 3}& \textbf{0}_{3\times 3}
		\end{array}
		\right],\\
		\textbf{H}^{Std}_{n+1}=\left[
		\begin{array}{cccc}
			\textbf{H}^{R,R^r}_{n+1}& \textbf{H}^{R,R^f}_{n+1}& \textbf{0}_{3\times 3}& \textbf{0}_{3\times 3}\\
			\textbf{H}^{p,R^r}_{n+1}& \textbf{0}_{3\times 3}& \textbf{H}^{p,p^r}_{n+1}& \textbf{H}^{p,p^f}_{n+1}\\
		\end{array}
		\right],\\
	\end{array}
	\label{FGH_std}
\end{equation}
where
$$\begin{aligned}
	&\textbf{H}^{R,R^r}_{n+1}=\textbf{H}^{p,p^r}_{n+1}=-({\textbf{R}}_{n+1|n}^r)^T,\\
	&\textbf{H}^{R,R^f}_{n+1}=\textbf{H}^{p,p^f}_{n+1}=({\textbf{R}}_{n+1|n}^r)^T,\\
	&\textbf{H}^{p,R^r}_{n+1}=(\textbf{R}^r_{n+1|n})^T(\textbf{p}^f_{n+1|n}-\textbf{p}^r_{n+1|n})^{\land},
\end{aligned}$$
$\textbf{F}_n$ and $\textbf{G}_n$ are  the Jacobians of process model evaluated by the state $\textbf{X}_{n|n}$ and the odometry $(\textbf{R}_n^u,\textbf{p}_n^u)$, respectively, and $\textbf{H}^{Std}_{n+1}$ evaluated at $\textbf{X}_{n+1|n}$ is the Jacobian of observation $\textbf{y}$ defined in (\ref{Y}).

\subsection{Observability Analysis for Standard EKF}
\begin{myTheo}
	For Std-EKF, the unobservable subspace $\mathcal{\hat{N}}^{Std}$,   is a proper subspace of $\mathcal{\breve{N}}^{Std}$, where
	\begin{equation}		\mathcal{\hat{N}}^{Std}=\mathop{\text{span}} _{col.}[\textbf{0}_{3\times 3},\textbf{0}_{3\times 3},\textbf{I}_3,\textbf{I}_3]^T,
	\end{equation}
	and 
	\begin{equation}
		\mathcal{\breve{N}}^{Std}=\mathop{\text{span}} _{col.}\left[
		\begin{array}{cccccc}
			\textbf{0}_{3,3}& \textbf{I}_{3}\\
			\textbf{0}_{3,3}& \textbf{I}_{3}\\
			\textbf{I}_{3}& (\textbf{p}^r_0)^{\land}\\
			\textbf{I}_{3}& (\textbf{p}^f)^{\land}
		\end{array}
		\right],
	\end{equation}
	$\textbf{p}^r_0$ and $\textbf{p}^f$ are respectively the true positions of initial robot and object feature.
	And the dimension of $\mathcal{\hat{N}}^{Std}$ is 3, while the dimension of $\mathcal{\breve{N}}^{Std}$ is 6.
	\label{Std_OC}
\end{myTheo}
\begin{proof}
    See Appendix C. 
\end{proof}

According to Theorem \ref{Std_OC}, due to this improper linearization for object based SLAM, standard EKF mistakenly takes spurious information into estimation, leading to overconfident estimate (inconsistency).


\section{Simulations}
\begin{figure}[t]
	\centering
	\includegraphics[width=.4\textwidth]{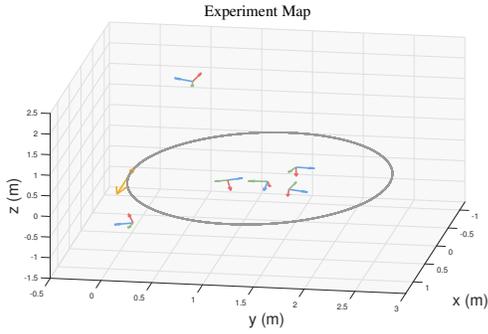}
	\caption{Simulation environment: 6 object features (the poses are shown as red-green-blue arrows) in a 3D environment, robot moves on the circle (the yellow arrow shows the initial heading of the robot).}
	\label{Map}
\end{figure}

\begin{figure*}[t]
	\begin{minipage}[h]{1\textwidth}
		\centering
		\includegraphics[width=.27\textwidth]{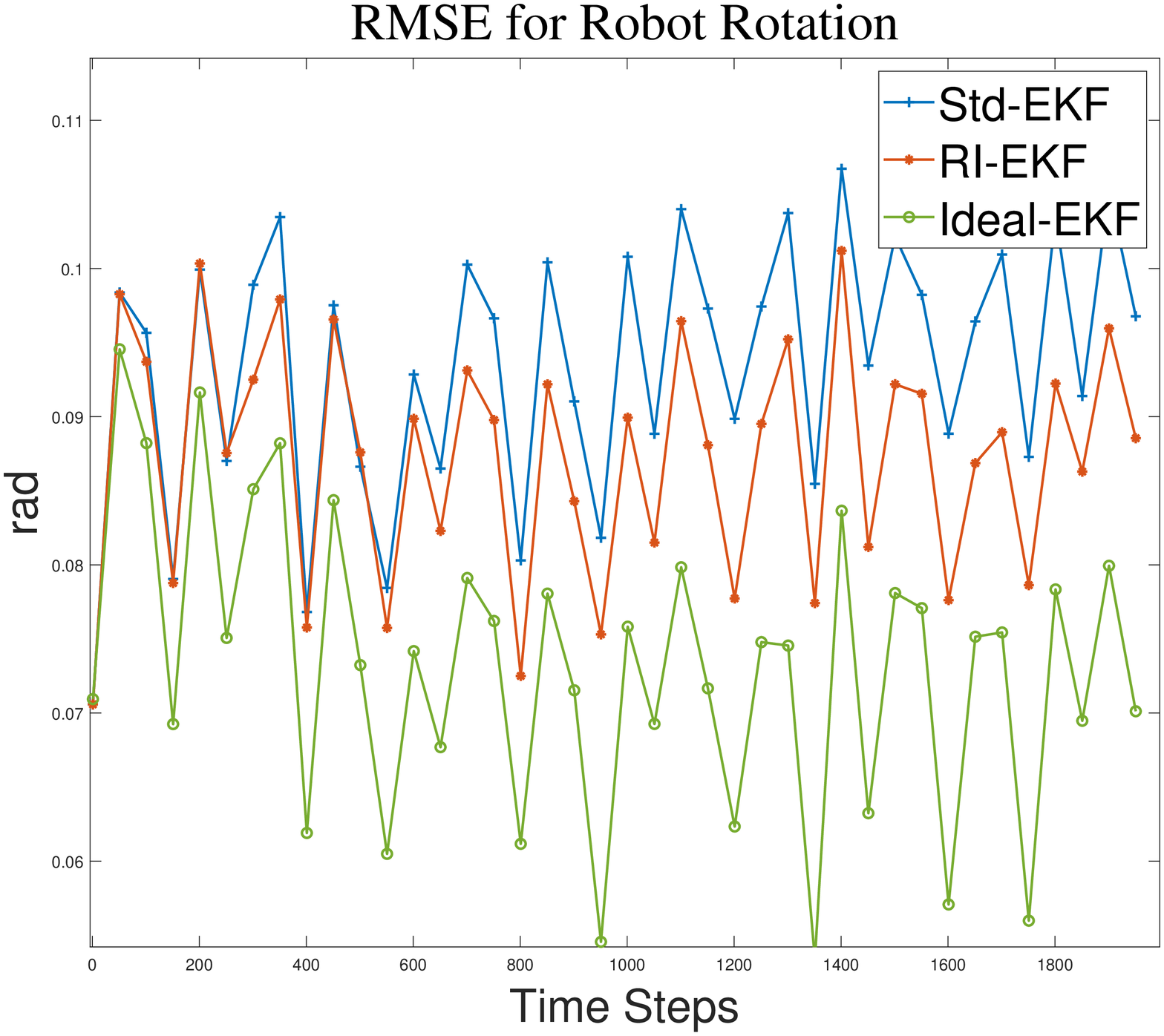}
		\includegraphics[width=.27\textwidth]{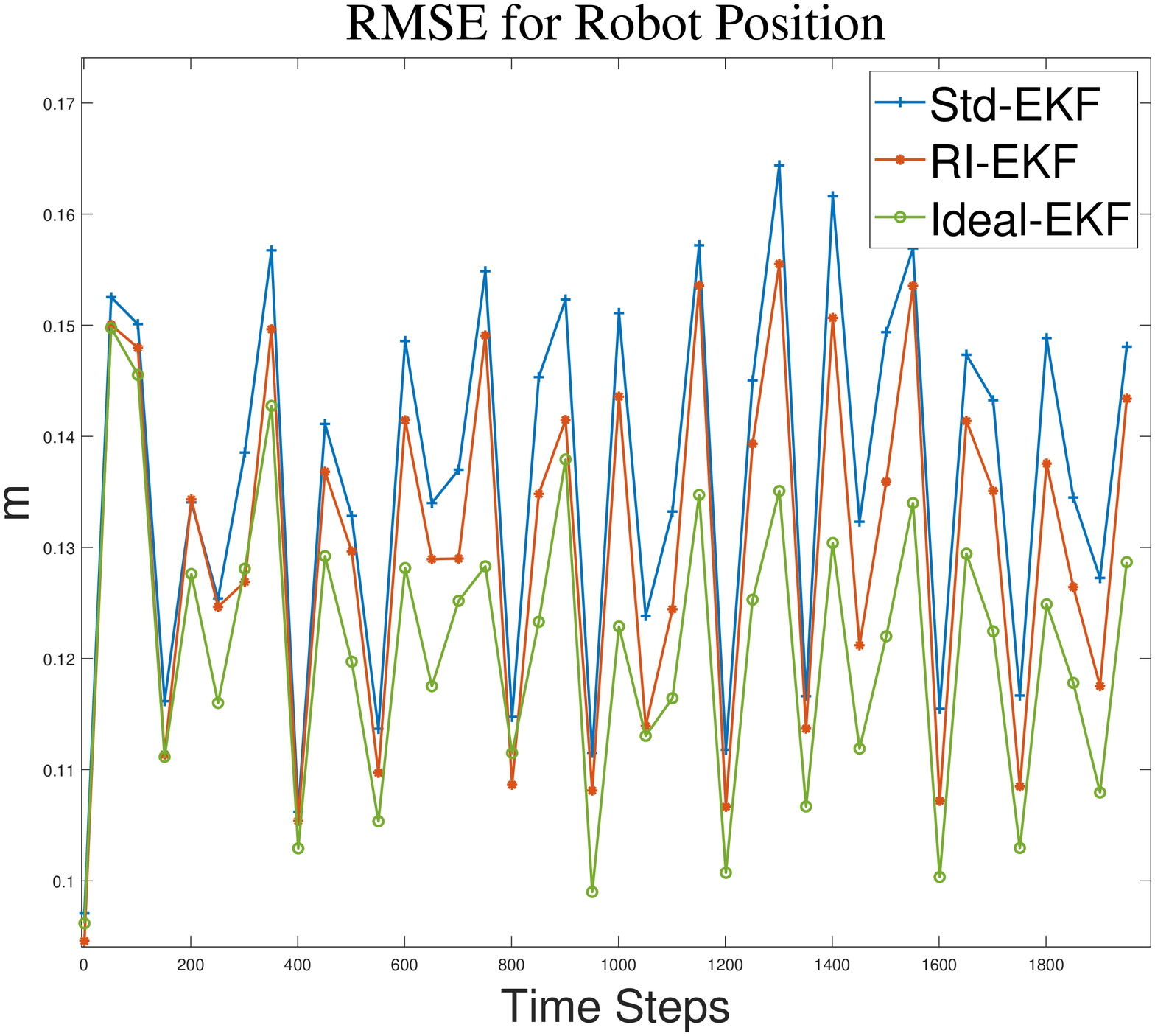}
		\includegraphics[width=.27\textwidth]{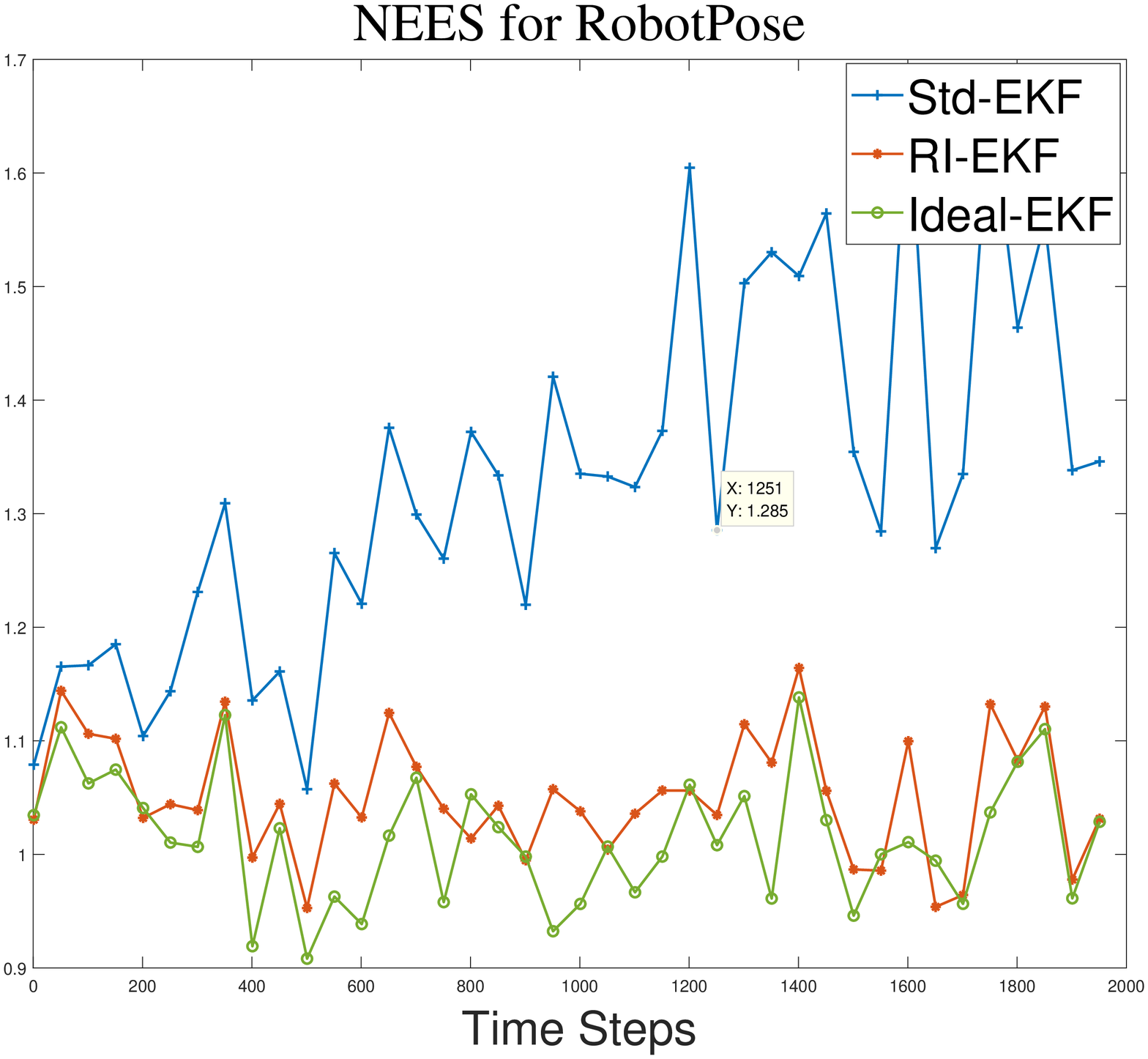}
	\end{minipage}
	\begin{minipage}[h]{1\textwidth}
		\centering
		\includegraphics[width=.27\textwidth]{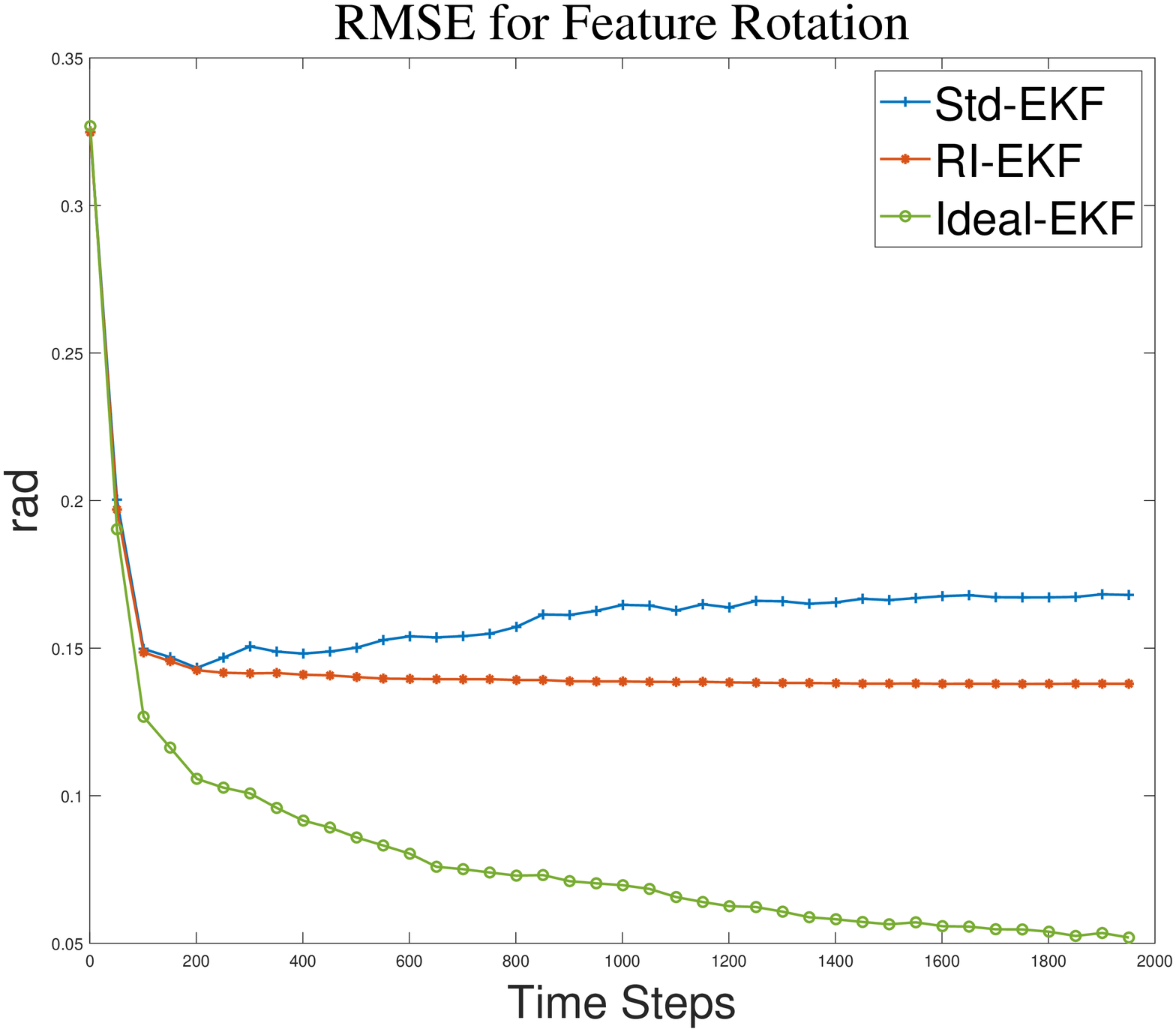}
		\includegraphics[width=.27\textwidth]{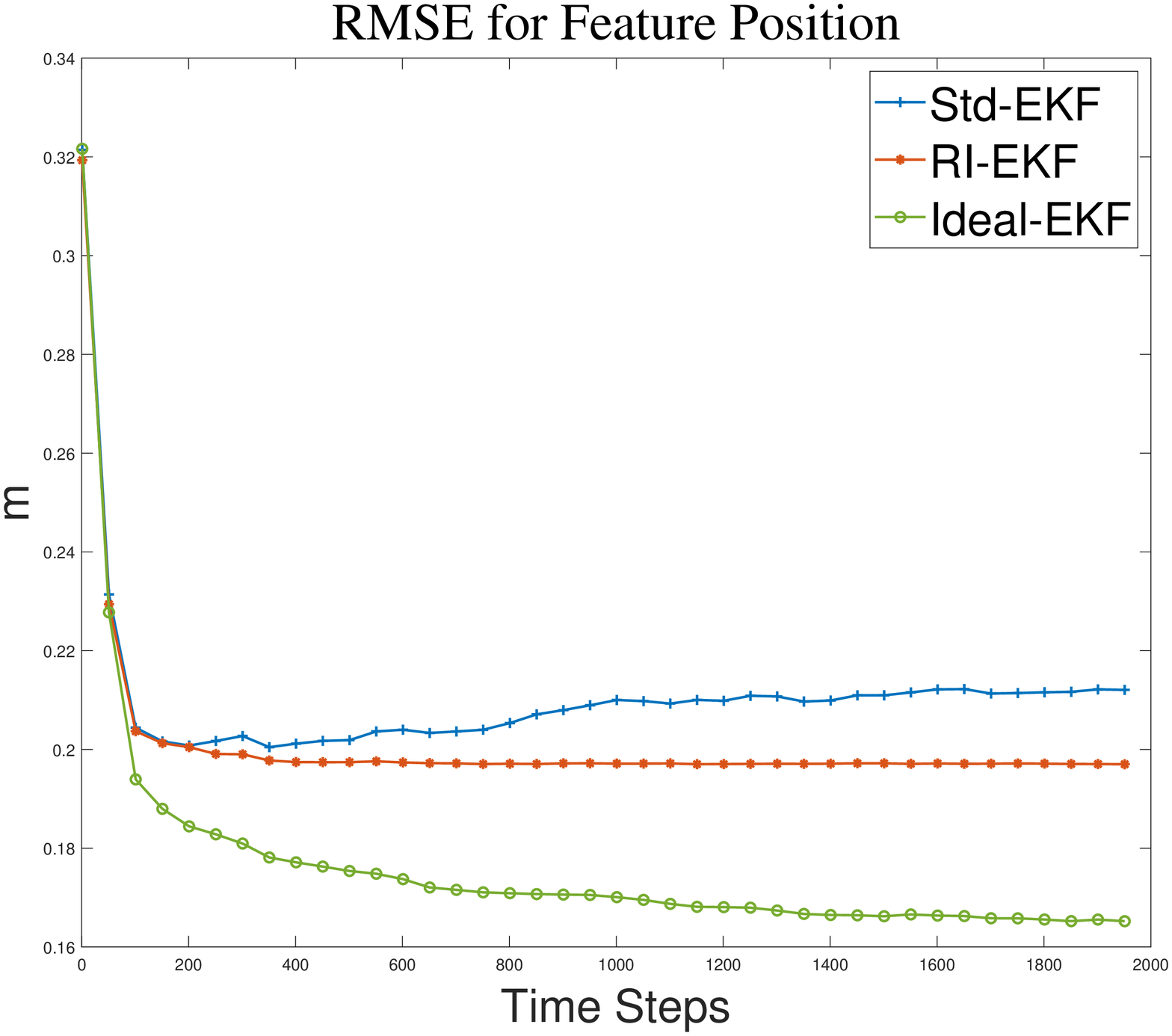}
		\includegraphics[width=.27\textwidth]{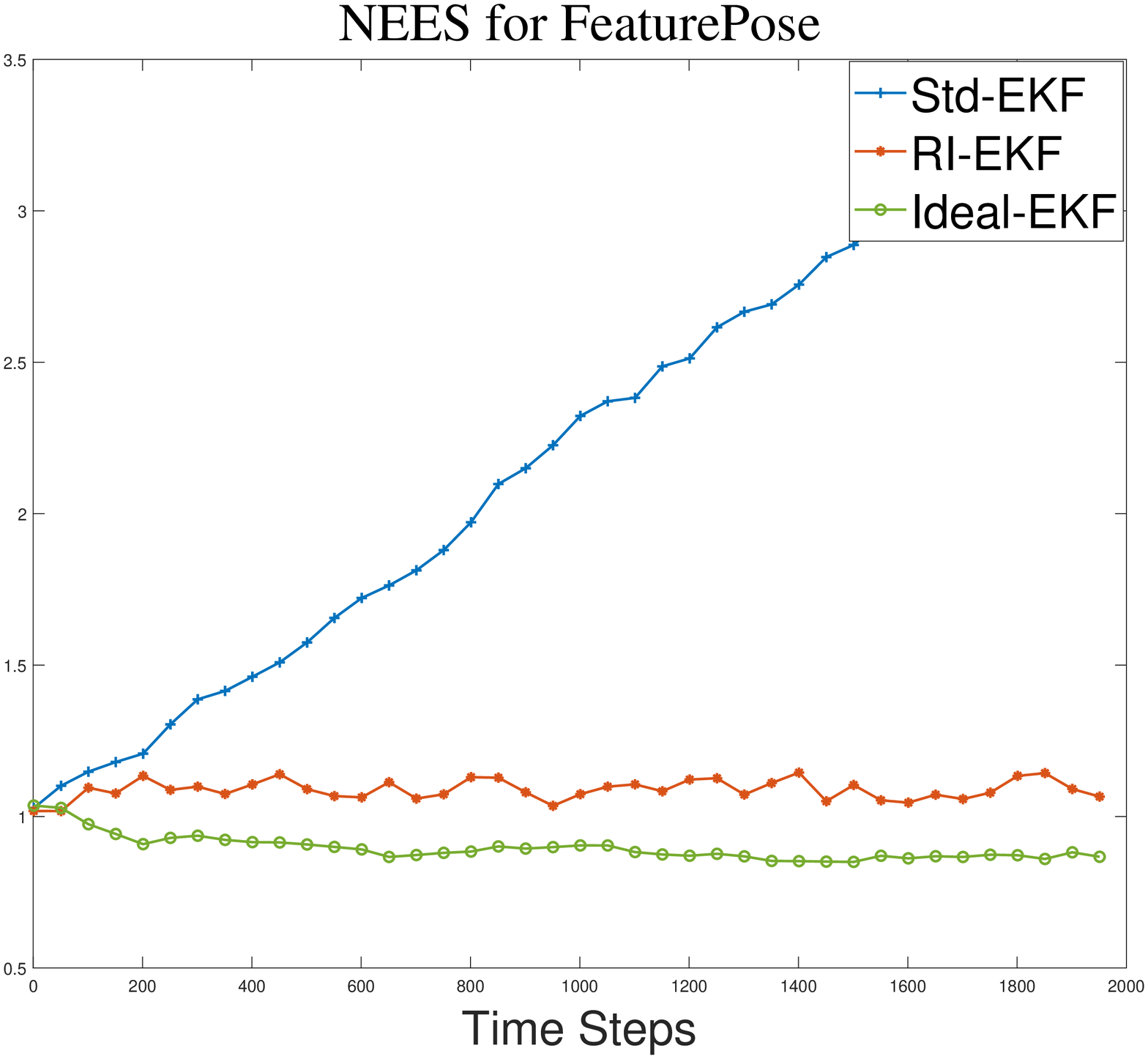}
	\end{minipage}
	\caption{Accuracy (RMSE) and consistency (NEES) of Std-EKF, RI-EKF, and Ideal-EKF in simulations.}
	\label{NEES}
\end{figure*}

\begin{table}
	
	\centering
	\caption{Simulation Results  of RI-EKF, Std-EKF and Ideal-EKF}
	\begin{tabular}{|c|c|c|c|c|c|c|c|c|c|}
		\hline
		& Std-EKF &RI-EKF &  Ideal-EKF\\
		\hline
		&\multicolumn{3}{|c|}{RMSE}\\
		\hline
		Robot Rotation (rad) & 0.0919& 0.0851 & 0.0719 \\
		Robot Position (m) & 0.1363& 0.1306 & 0.1215 \\
		Feature Rotation (rad) & 0.0271& 0.0231 & 0.0134 \\
		Feature Position (m)& 0.0365& 0.0343  & 0.0307 \\
		\hline
		&\multicolumn{3}{|c|}{NEES}\\
		\hline
		Robot Rotation & 1.5623 & 1.0622& 0.9852 \\
		Robot Position & 1.2711& 1.0304 & 1.0560 \\
		Robot Pose & 1.3435& 1.0592 & 1.0279 \\
		Feature Rotation & 3.5999& 1.1539 & 0.9091 \\
		Feature Position & 1.3457& 0.9990 & 1.0378 \\
		Feature Pose & 2.3425& 1.0849 & 0.9981 \\
		\hline
	\end{tabular}

	\label{Perf_Std_RI}
	
\end{table}

In this section, we compare our proposed RI-EKF with standard EKF (Std-EKF) and Ideal-EKF (a variant of the Std-EKF where Jacobians are evaluated at the ground truth). It should be noted that Ideal-EKF is impossible to be applied in the real scenario, since the ground truth is not available. It is just used to explain the influence of observability constraints on inconsistency. We use Normalized Estimation Error Squared (NEES) indicator to evaluate the consistency of an estimation method
\begin{equation}
	\text{NEES}=\frac{1}{m\times d}\sum_{i=1}^{m}\textbf{e}_i^T\textbf{P}_i^{-1}\textbf{e}_i,
\end{equation}
where $m$ is the number of samples, and $\textbf{e}_i$ is a $d$  dimensional error sample vector, which is estimated to be a zero mean Gaussian with a $d\times d$ covariance matrix $\textbf{P}_i$. NEES should approximately equal to 1 for large $m$, if the estimator is consistent. In addition, root mean squared error (RMSE) is used to evaluate the accuracy of each estimator.

To compute NEES, it is worth noting that for our proposed RI-EKF, the estimated covariance is corresponding to the nonlinear error defined in (\ref{eq:nonlinear_error}) instead of the standard error in the vector space. However, for a fair comparison, we still use the standard error to compute the RMSE of RI-EKF.

\subsection{Settings}

The simulation environment and the robot trajectory are shown in Fig. \ref{Map}. There are 6 object features in the environment and their poses are represented by the red-green-blue arrows. The robot moves along a circle 25 times (total length: 200m) with a constant linear velocity 0.1 m/s and a constant angular velocity $\pi /40$ rad/s. The robot is able to measure the relative poses of object features lying in a range of $[0.5 \text{m},2 \text{m}]$ around it. The units of rotation and position are radian (rad) and meter (m), respectively. The covariance matrices of odometry noise and observation noise in (\ref{ProcessModel}) and (\ref{ObsModel}) are set to 
$\bm{\Sigma}_n=\text{diag}(0.1^2,0.1^2,0.1^2, 0.1^2,0.1^2,0.1^2)$ and $\bm{\Omega}_n=\text{diag}(0.1^2,0.1^2,0.1^2, 0.1^2 0.1^2 0.1^2)$.

\subsection{Result and Analysis}

We conducted 50 Monte-Carlo simulations, i.e. $m=50$. The NEES and RMSE results are shown in Fig. \ref{NEES} and Table  \ref{Perf_Std_RI}. Fig. \ref{NEES} shows the RMSE and NEES results for robot pose and feature pose every 50 steps. Table \ref{Perf_Std_RI} lists the average RMSE and NEES for rotation and position error (in rad and m respectively) in the last time step. The results show that in this experiment, RI-EKF and Ideal-EKF perform better than Std-EKF in terms of both accuracy and consistency. Based on the comparison of Std-EKF and Ideal-EKF, we can see that the inconsistency of Std-EKF mainly comes from the inaccuracy of linearization points. And according to the analysis in Section \ref{Ob_anal}, these linearization points in Std-EKF break the observability constraints,  leading Std-EKF to obtain spurious information from unobservable subspace. As a result, its estimation will be more inaccurate and its estimated covariance is smaller than the actual uncertainty, and become more and more inconsistent over time. In contrast, RI-EKF remains consistent (NEES $\approx 1$) in a longer duration, behaving like Ideal-EKF. The analysis for RI-EKF in Section \ref{Ob_anal}  indicates that RI-EKF naturally maintains the observability constraints, as in Ideal-EKF. And each Jacobians are evaluated at the latest estimate. These make the results of RI-EKF more reliable than those of Std-EKF.

\section{Real Data Experiments}
In this section, we test our algorithm on a real dataset YCB-Video \cite{YCB} and compare it with Std-EKF, DVO \cite{DVO} and ORB-SLAM2 \cite{orbslam2} to show its effectiveness.  All of these algorithms are fed by the RGB-D images from YCB-Video. Four sequences (0019, 0036, 0041, and 0049), which have relatively long trajectory, are selected to be used in the experiments (Fig. \ref{RealData}).  Different from simulation, the data collected from real world may have many outliers.  The matches of point clouds in Sequence 0019 are very accurate, but in other sequences, some object observations are very inaccurate, as shown in the last three images in the lower row of Fig. \ref{RealData}. 

In order to make the algorithms robust, we need to detect and remove outliers. In addition, there is no information about odometry in these data sequences. To apply our algorithms, we make a simple constant velocity assumption for these data sequences which are obtained at low speed. 

\subsection{Dataset and Object Detection}
YCB-Video dataset contains 21 objects with various textures from YCB objects. There are 92 RGB-D videos used for training and testing object detection, in which 80 videos are used for training and 2949 keyframes from the rest 12 videos are used for testing. Besides, 80000
synthetic images are released for training. There are many scenes of stacking objects with partial occlusion, as is shown in Fig. \ref{RealData}.

\subsection{Observation of Object Features}
To get the observation of the object features, in the  front-end, we utilize the algorithm called REDE from  \cite{REDE}, which is an end-to-end object pose estimator using RGB-D data as inputs. In YCB dataset, the results of the pose estimator can realize $98.9\%$ recall under the metric of average ADD \cite{add}. The outputs of REDE are directly fed into (\ref{ObsModel}) as observations.

\subsection{Constant Velocity Assumption}
Since the data are collected from a low speed camera, we assume that the camera is moving at constant velocity. Here we just take a very simple method to obtain the odometry. We assume the angular velocity is zero with noise. And expected linear velocity is the average of previous estimations.
\subsection{Outlier Removal}
Suppose there is an object feature observed in the $n+1$ step, and $\textbf{Z}\in \mathbb{SO}(3)\times \mathbb{R}^3$ is its observation. Before update the state vector, we will first compute $\textbf{y}=\textbf{Z}\boxminus h(\textbf{X}_{n+1|n},\textbf{0})$ in (\ref{Y}). According to EKF framework, 
\begin{equation}
	\begin{array}{lll}
		\textbf{y}\approx \textbf{H}_{n+1}\bm{\xi}_{n+1|n}+\textbf{v}_{n+1}\\
	\end{array}
\end{equation}
where $\textbf{H}_{n+1}$ is the Jacobian of observation function evaluated at $\textbf{X}_{n+1|n}$, and $\textbf{v}_{n+1}\sim N(\textbf{0},\bm{\Omega}_{n+1})$ is the noise. If the estimation is accurate, then $\textbf{y}$ should form a zero mean Gaussian distribution with covariance 
\begin{equation}
	\textbf{P}_y=\textbf{H}_{n+1}\textbf{P}_{n+1|n}(\textbf{H}_{n+1})^T+\bm{\Omega}_{n+1}.
	\label{Py}
\end{equation}
Therefore, if all the elements of $\textbf{y}$, i.e. $\textbf{y}(k),\ k=1,\cdots, 6$, are in its $3\sigma$ bound, i.e.  $$|\textbf{y}(k)|<3\sqrt{\textbf{P}_y(k,k)},\ k=1,\cdots 6,$$
then we will use $y$ to update the state. Otherwise, the observation will be considered as an outlier. The EKF methods with this outlier removal are called Robust EKF methods in the following. 

\begin{figure*}[t]
	\centering
	\includegraphics[width=1\textwidth]{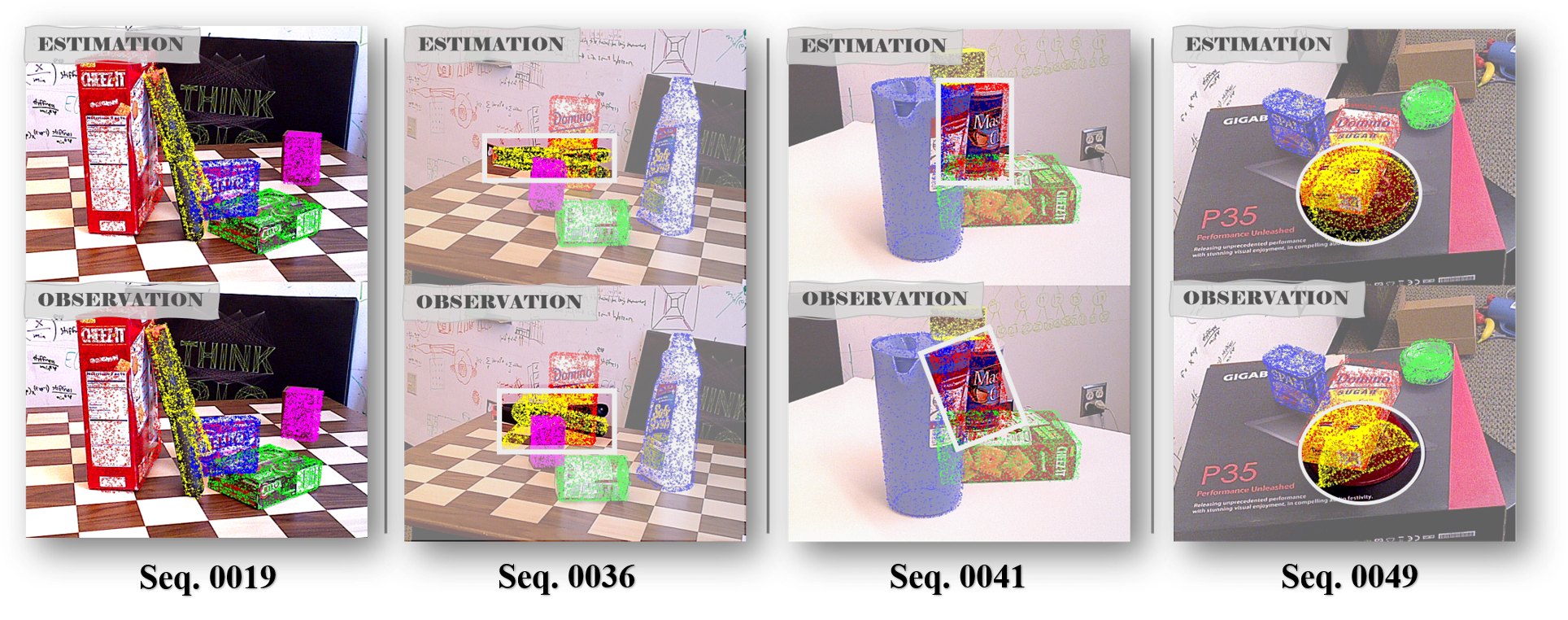}
	\caption{Sample images from the four sequences in YCB-Video Dataset \cite{YCB} used in the experiments. For each column, the lower row shows an image with feature observations while the upper row shows the final estimate from our proposed method.}
	\label{RealData}
\end{figure*}

\begin{table}
	
	\centering
	\caption{Average RMSE* results from four RGB-D sequences of YCB-Video dataset \cite{YCB}.}
	\begin{tabular}{|c|c|c|c|c|}
		
		\hline
		
		& 0019 & 0036 & 0041 & 0049 \\
		\hline
		
		DVO & .015/.029 & .021/.055 & .023/.038 & .010/.046 \\
		ORB-SLAM2 & .012/.030 & .017/.044 & .014/.028 & \textbf{.004}/{.031} \\
		Std-EKF & .007/.009 & .012/.017 & .008/.010 & .201/.252 \\
		Robust Std-EKF & .007/.009 & .015/.021 & .008/.010 & .011/.023 \\
		RI-EKF & \textbf{.004}/\textbf{.006} & \textbf{.004}/.007 & \textbf{.006}/\textbf{.007} & .156/.205 \\
		Robust RI-EKF & \textbf{.004}/\textbf{.006} & \textbf{.004}/\textbf{.006} & \textbf{.006}/\textbf{.007} & {.007}/\textbf{.022}\\
		\hline
	\end{tabular}
	{*RMSE of Robot Position (m)/RMSE of Robot Rotation (rad)}
	
	\label{comparisons}
	
\end{table}
\subsection{Results and Analysis}
We compare our algorithm with DVO \cite{DVO}, ORB-SLAM2 \cite{orbslam2} and Std-EKF using the four data sequences. The  average RMSE for robot rotation and position are shown in Table \ref{comparisons}. 
In general, Robust RI-EKF performs the best on these four data sequences as listed in Table \ref{comparisons}. Although both DVO and ORB-SLAM2 utilize all the features in the whole trajectory, they only exploit low-level features (point features). However, the object features considered in our methods have broader perspective loop closures, longer feature track, and more intrinsic constraints information between low-level features. Hence, DVO and ORB-SLAM2 are not fairly comparable. This could explain why the filter based Robust RI-EKF outperforms these two optimization based methods.


The comparison of Robust RI-EKF, Robust Std-EKF, RI-EKF, and Std-EKF on Sequence 0049 shows that robust methods are far better on such inaccurate data sequences. 
We also note that some RMSE of Robust Std-EKF are larger than that of Std-EKF, especially on Sequence 0036. This is mainly caused by the inconsistency of Std-EKF which results in mistaken deletion for correct data. 

Fig. \ref{3S_RI_Std} shows the errors of robot pose estimates and the 3$\sigma$ bounds of each component for Robust RI-EKF and Robust Std-EKF on Sequence 0036, respectively. Robust RI-EKF perfoms well in terms of consistency, while Robust Std-EKF underestimates the uncertainty of the state in the latter half of the sequence. This can finally result in larger errors. In contrast, the estimates by Robust RI-EKF are more reliable, making our outlier removal more effective.

Fig. \ref{RMSE_obj} shows the RMSE of objects in every 100 steps. They illustrate that Robust RI-EKF also broadly generates more accurate estimates on the object poses than Robust Std-EKF.

In general, from the real data experiments, we can see Robust RI-EKF can generate good estimation results. In Fig. \ref{RealData}, the upper row images show the estimated objects from Robust RI-EKF, which significantly improve the corresponding observations shown in the lower row images for Sequences 0036, 0041 and 0049. 
Although the difference for Sequence 0019 is not significant, it can be seen in Fig. \ref{Display2} that the estimation results is more closer to the true object as compared with the observations. 
{\tiny }

\begin{figure*}[t]
	\centering
	\begin{minipage}[h]{1\textwidth}
		\centering
		\includegraphics[width=.16\textwidth]{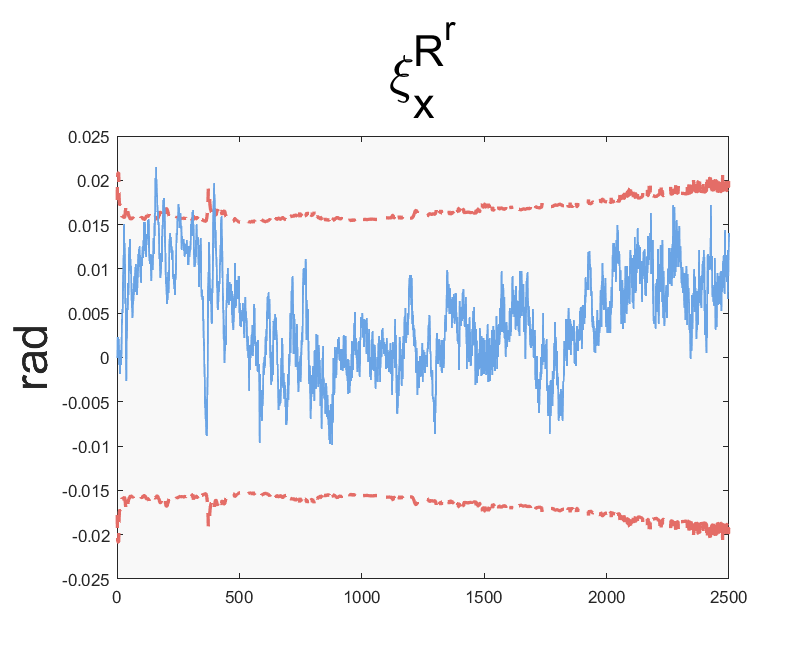}
		\includegraphics[width=.16\textwidth]{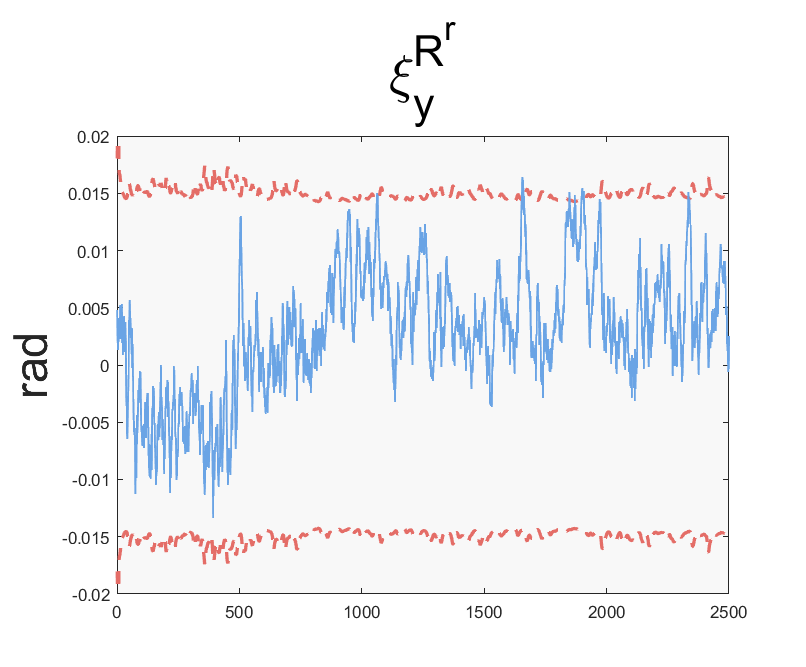}
		\includegraphics[width=.16\textwidth]{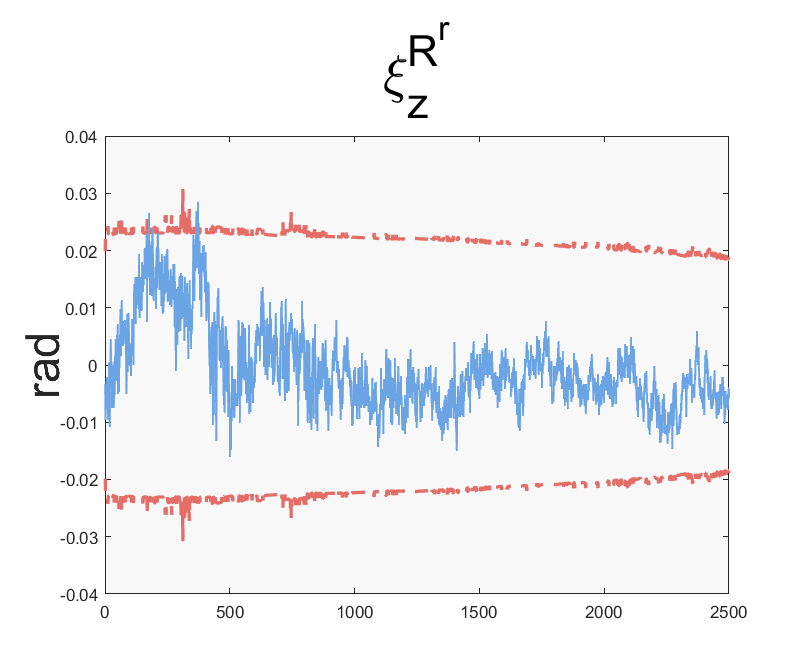}
		\includegraphics[width=.16\textwidth]{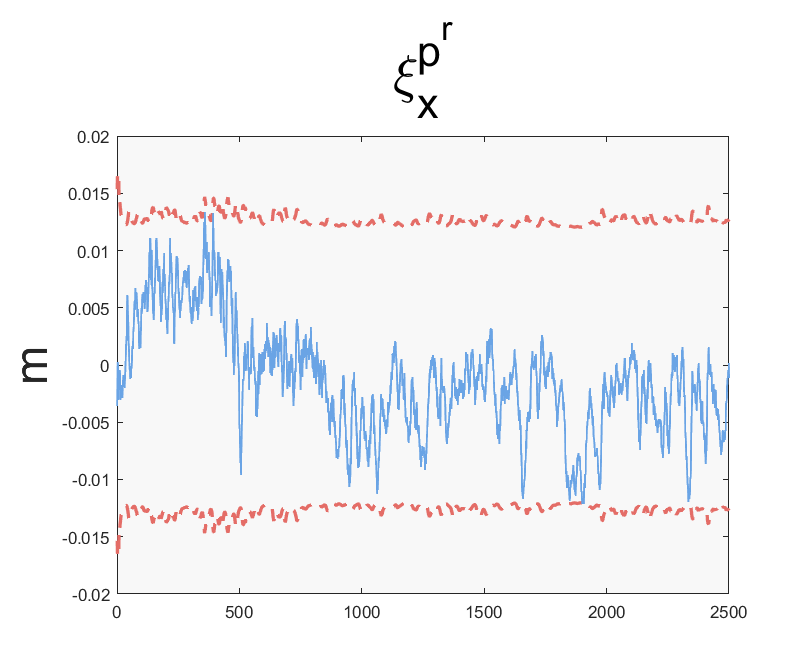}
		\includegraphics[width=.16\textwidth]{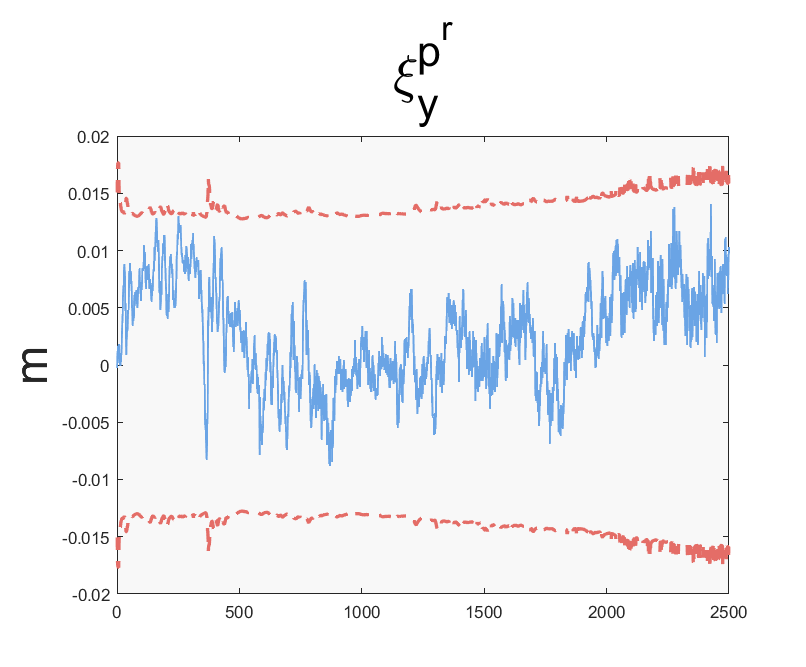}
		\includegraphics[width=.16\textwidth]{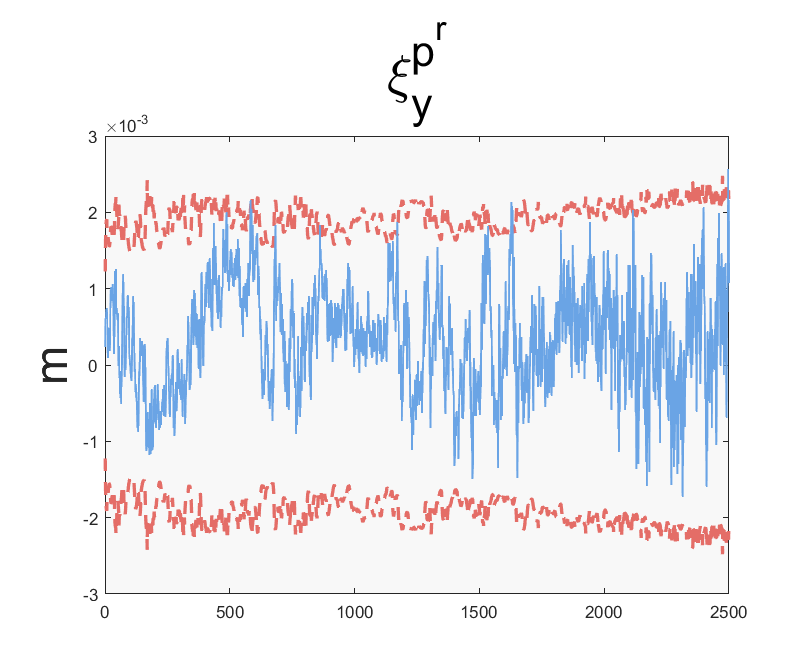}
	\end{minipage}
	\centering
	\begin{minipage}[h]{1\textwidth}
		\centering
		\includegraphics[width=.16\textwidth]{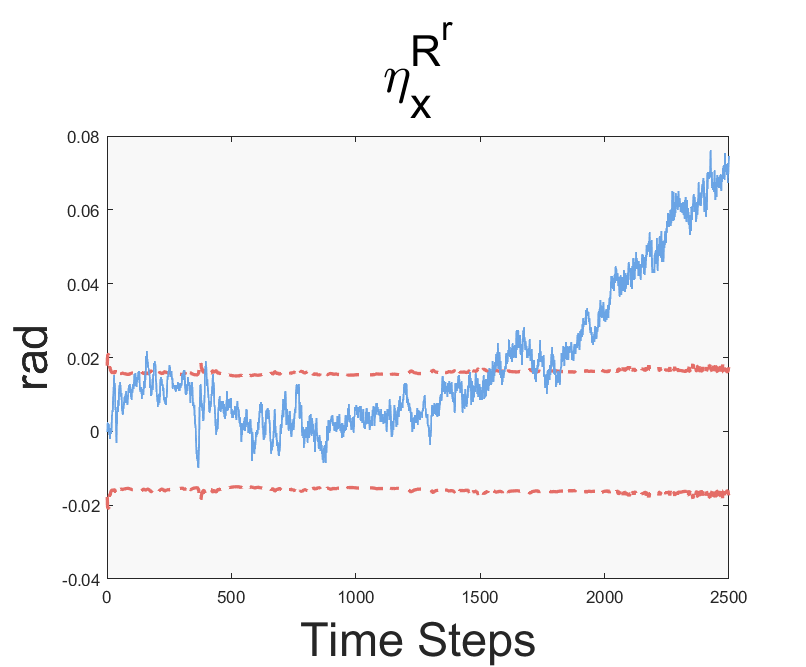}
		\includegraphics[width=.16\textwidth]{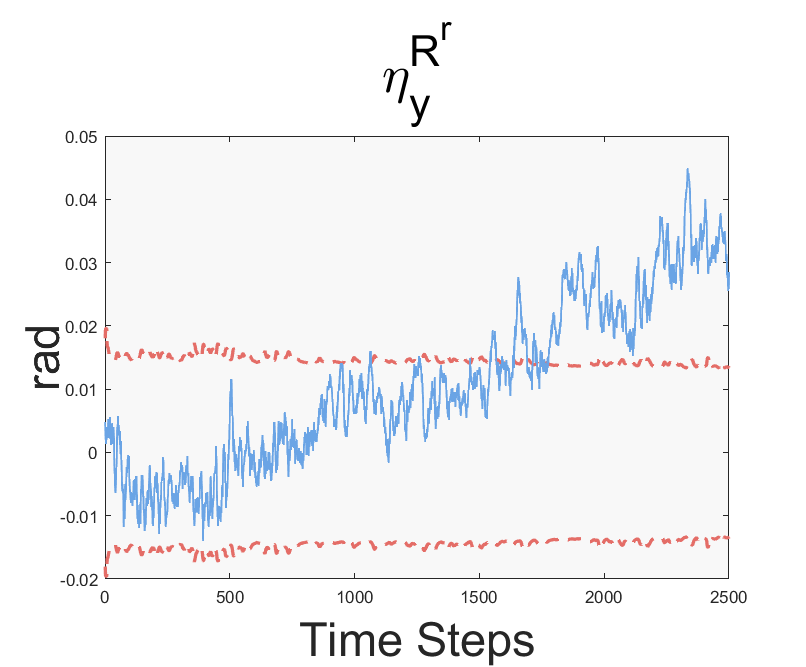}
		\includegraphics[width=.16\textwidth]{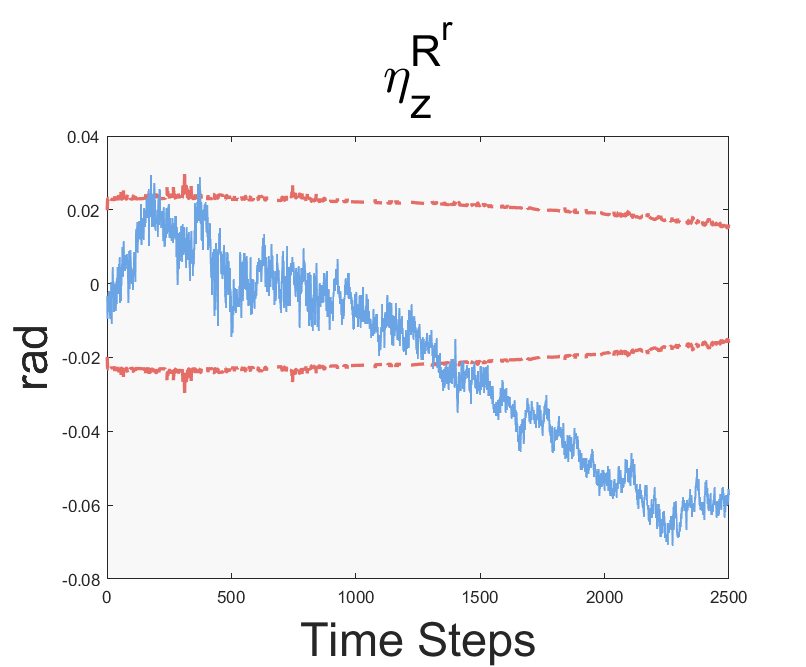}
		\includegraphics[width=.16\textwidth]{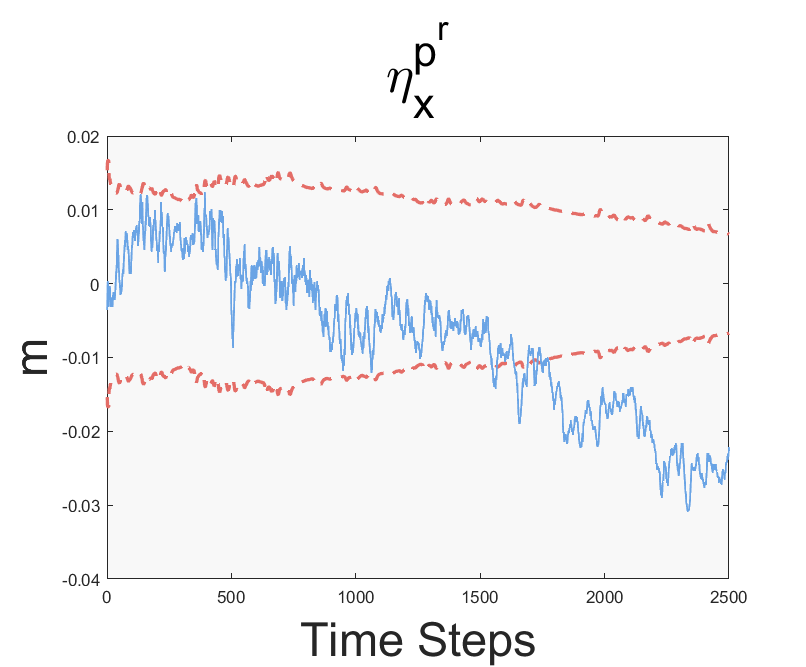}
		\includegraphics[width=.16\textwidth]{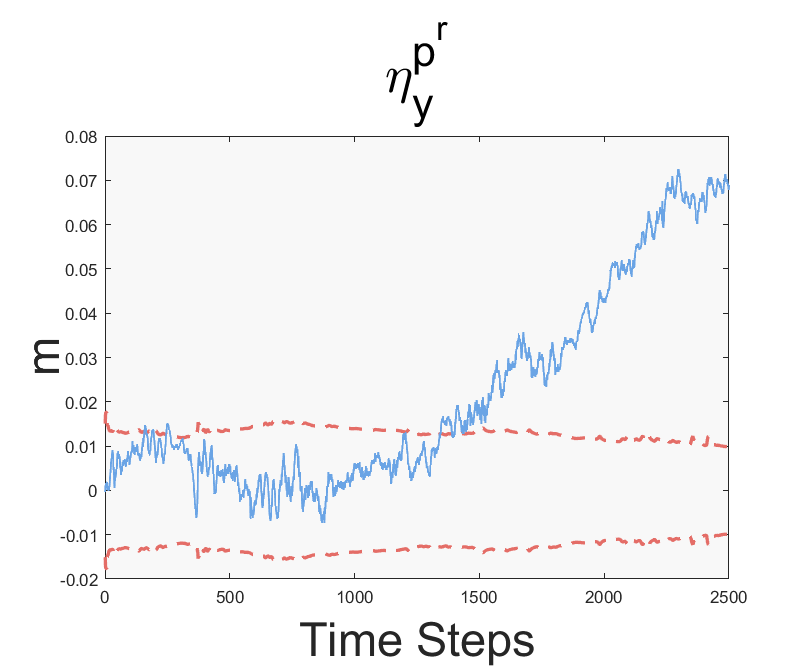}
		\includegraphics[width=.16\textwidth]{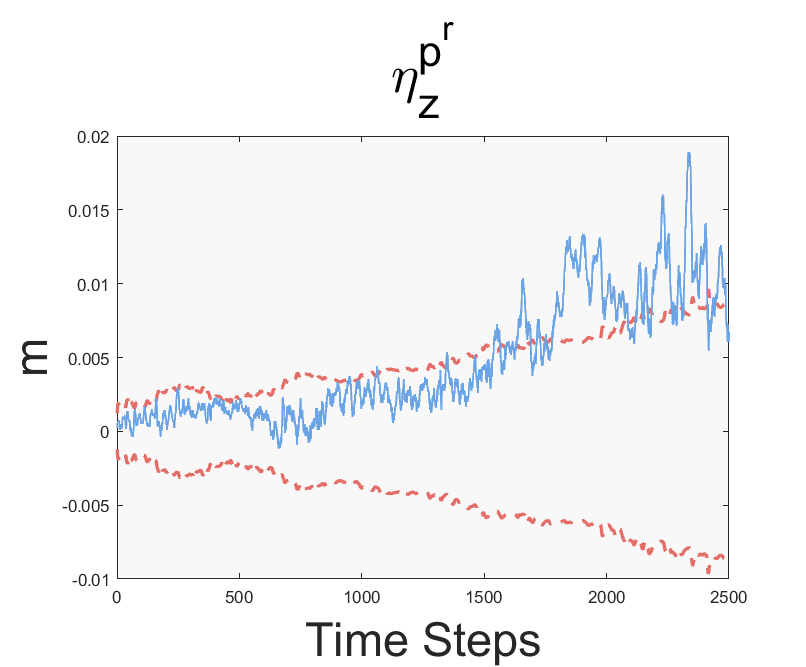}
	\end{minipage}
	\caption{Robot pose estimate errors and the corresponding 3$\sigma$ bounds for Sequence 0036: the upper six figures are for Robust RI-EKF, the lower six figures are for Robust Std-EKF.}
	\label{3S_RI_Std}
\end{figure*}



\begin{figure}[t]
	\centering
	\begin{minipage}[h]{.5\textwidth}
		\centering
		\includegraphics[width=1\textwidth]{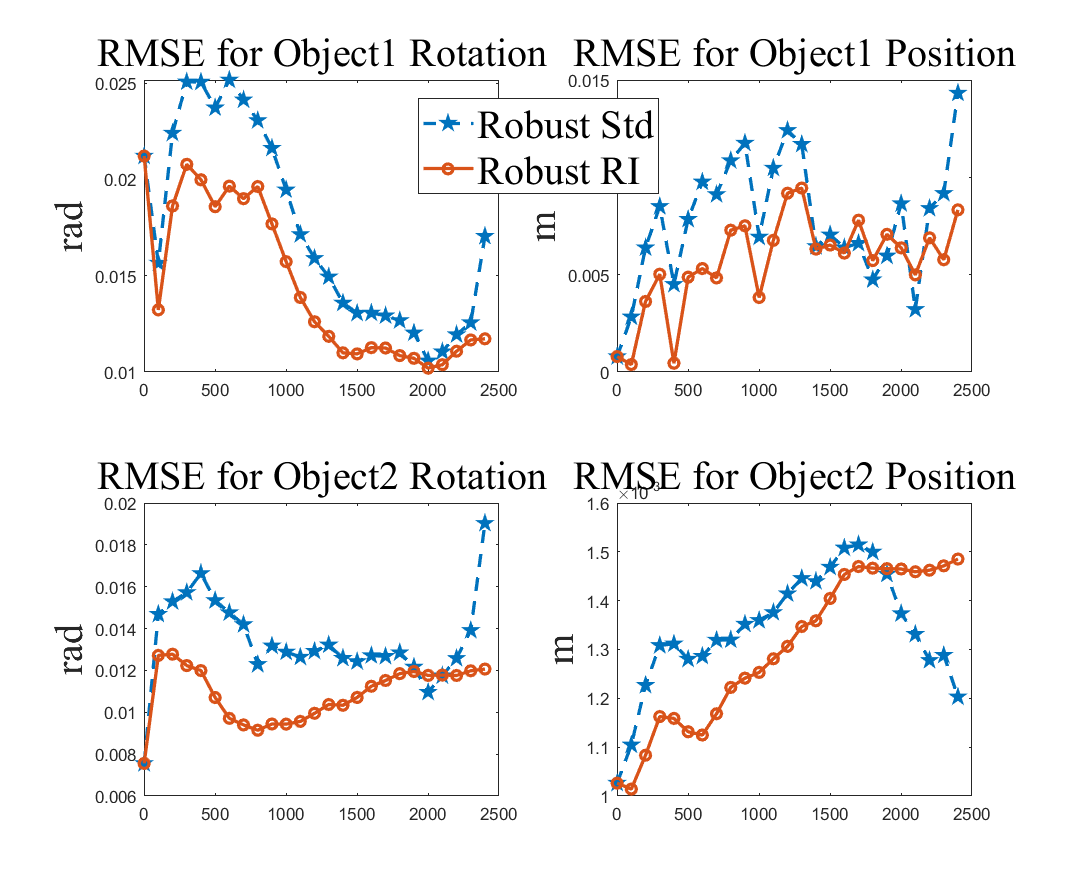}
	\end{minipage}
	\centering
	\begin{minipage}[h]{.5\textwidth}
		\centering
		\includegraphics[width=1\textwidth]{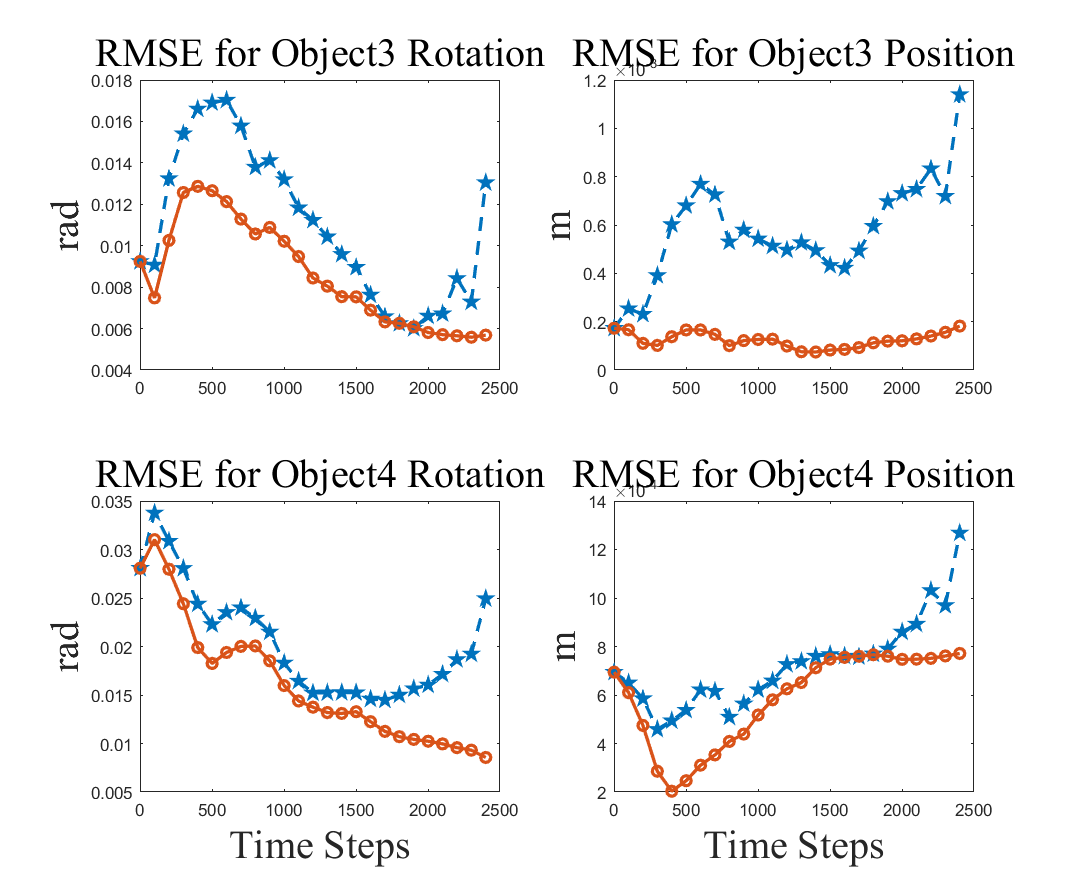}
	\end{minipage}
	\caption{RMSE of objects in Sequence 0041 for Robust RI-EKF and Robust Std-EKF. (Object1: Master Chef Can; Object2: Cracker Box; Object3: Pitcher Base; Object4: Foam Brick)}
	\label{RMSE_obj}
\end{figure}

\begin{figure}[t]
	\centering
	\includegraphics[width=.4\textwidth]{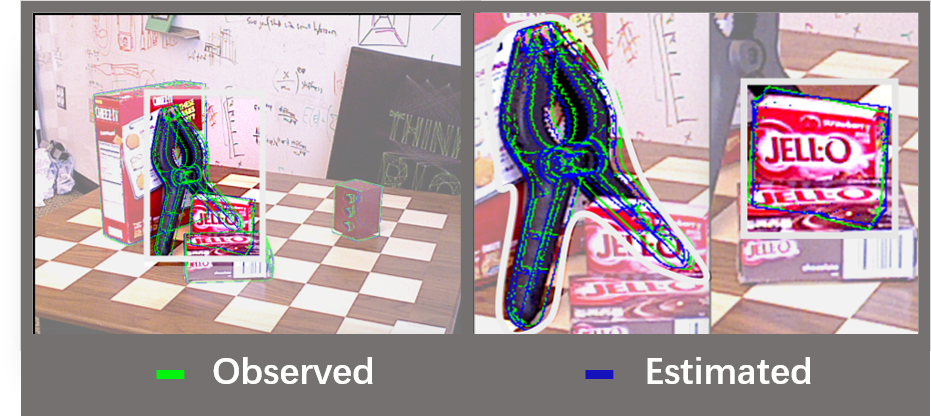}
	\caption{A sample result of Robust RI-EKF for Sequence 0019.}
	\label{Display2}
\end{figure}

\section{Conclusion}

In this work, we propose a right invariant EKF (RI-EKF) algorithm for object based SLAM, where object features are represented by 3D poses and are estimated together with the latest robot pose. From theoretical analysis, we prove that our RI-EKF generated by the proposed Lie group can automatically maintain the correct observability properties. This is different from standard EKF with object features that does not have the correct observability property. Results from simulations and real data experiments confirm the good performance of the proposed RI-EKF algorithm.

In this paper, we focus on the SLAM back-end and assume the objects observed are within a given database such that they can be detected and matched relatively easily from the SLAM front-end. In the future, we will investigate the more challenging object based SLAM problem where the objects in the environments are more general and may not belong to a known database.



\section*{APPENDIX}
In this appendix, we provide the mathematical derivation for the feature initialization/augmentation method in Section \ref{sec:initialization}, and the proofs of the two theorems in Section \ref{Ob_anal}.
\subsection{Mathematical Derivation for Initialization Method}\label{Math}
Note that
\begin{equation}
	\begin{array}{rll}
		\exp(\bm{\xi}^{R^r})&=\textbf{R}^r(\hat{\textbf{R}}^r)^T,\\
		\exp(\bm{\xi}^{R^{f}})&=\textbf{R}^{f}(\hat{\textbf{R}}^{f})^T,\\
		\bm{\xi}^{p^r}&\approx \textbf{p}^r-\exp(\bm{\xi}^{R^r})\hat{\textbf{p}}^r,\\
		\bm{\xi}^{p^{f}}&\approx \textbf{p}^{f}-\exp(\bm{\xi}^{R^r})\hat{\textbf{p}}^{f}.
	\end{array}
\end{equation}
And considering the observation model of the new feature, $(\textbf{R}^{f_{\text{new}}},\textbf{p}^{f_{\text{new}}})$,
\begin{equation}\label{eq:obs_model_app}
	\begin{array}{rll}
		\textbf{R}^z&=\exp^{\mathbb{SO}(3)}((\textbf{v}^R)^{\land})(\textbf{R}^r)^T\textbf{R}^{f_{\text{new}}},\\
		\textbf{p}^z&= (\textbf{R}^r)^T(\textbf{p}^{f_{\text{new}}}-\textbf{p}^r)+\textbf{v}^p,
	\end{array}
\end{equation}
we have
\begin{equation}
	\begin{array}{rll}
		\textbf{R}^{f_{\text{new}}}&=\textbf{R}^r\exp^{\mathbb{SO}(3)}((-\textbf{v}^R)^{\land})\textbf{R}^z\\
		&=\exp^{\mathbb{SO}(3)}((\bm{\xi}^{R^r})^{\land})\hat{\textbf{R}}^r\exp^{\mathbb{SO}(3)}((-\textbf{v}^R)^{\land})\textbf{R}^z\\
		&\approx\exp^{\mathbb{SO}(3)}((\bm{\xi}^{R^r}-\hat{\textbf{R}}^r\textbf{v}^R)^{\land})\hat{\textbf{R}}^r\textbf{R}^z\\
	\end{array}
	\label{initialRp_app}
\end{equation}
and 
\begin{equation}
	\begin{array}{rll}
		\textbf{p}^{f_{\text{new}}}&=\textbf{p}^r+\textbf{R}^r\textbf{p}^z-\textbf{R}^r\textbf{v}^p\\&\approx\exp^{\mathbb{SO}(3)}(\bm{\xi}^R)\hat{\textbf{p}}^r\\
		&\ \ +\bm{\xi}^{p^r}+\exp^{\mathbb{SO}(3)}(\bm{\xi}^{R^r})\hat{\textbf{R}}^r\textbf{p}^z-\hat{\textbf{R}}^r\textbf{v}^p.
	\end{array}
\label{initialp_app}
\end{equation}
Since $\bm{\xi}\sim N(\textbf{0},\textbf{P})$ and $\textbf{v}\sim N(\textbf{0},\bm{\Omega})$, the expectation (or called estimation) of the new object feature, $(\hat{\textbf{R}}^{f_{\text{new}}},\hat{\textbf{p}}^{f_{\text{new}}})$, related to the observation $\textbf{Z}$ is given by
\begin{equation}
	\begin{array}{rll}
		\hat{\textbf{R}}^{f_{\text{new}}}&=\hat{\textbf{R}}^r\textbf{R}^z,\\
		\hat{\textbf{p}}^{f_{\text{new}}}&=\hat{\textbf{p}}^r+\hat{\textbf{R}}^r\textbf{p}^z.
	\end{array}\label{expectations_app}
\end{equation}

For their covariance, denote $\bm{\xi}^{R^{f_{\text{new}}}}$ and $\bm{\xi}^{p^{f_{\text{new}}}}$ as the corresponding components for $\textbf{R}^{f_{\text{new}}}$ and $\textbf{p}^{f_{\text{new}}}$ in the new error state vector, respectively. Notice that by the proposed Lie group structure, we have
\begin{equation}
	\begin{array}{rll}
		\exp^{\mathbb{SO}(3)}(\bm{\xi}^{R^{f_{\text{new}}}})&=\textbf{R}^{f_{\text{new}}}(\hat{\textbf{R}}^{f_{\text{new}}})^T,&\\
		\bm{\xi}^{p^{f_{\text{new}}}}&\approx 
		\textbf{p}^{f_{\text{new}}}-\exp^{\mathbb{SO}(3)}(\bm{\xi}^{R^r})\hat{\textbf{p}}^{f_{\text{new}}}.\\
	\end{array}\label{initialxi1_app}
\end{equation}
Then, combining equations (\ref{initialRp_app}), (\ref{initialp_app}), (\ref{expectations_app}), and  (\ref{initialxi1_app}), we obtain
\begin{equation}
	\begin{array}{rll}
		\bm{\xi}^{R^{f_{\text{new}}}}\approx \bm{\xi}^{R^r}-\hat{\textbf{R}}^r\textbf{v}^R\\
	\end{array}
\end{equation}
and
\begin{equation}
	\begin{array}{rll}
		\bm{\xi}^{p^{f_{\text{new}}}}\approx\bm{\xi}^{p^r}-\hat{\textbf{R}}^r\textbf{v}^p. 
	\end{array}
\end{equation}
Therefore, consider the state covariance before augmentation, 
$$\textbf{P}=\left[
\begin{array}{ccc}
	\textbf{P}^{R,R}& \textbf{P}^{R,p}\\
	\textbf{P}^{p,R}& \textbf{P}^{p,p}
\end{array}
\right],$$ where $\textbf{P}^{i,j}=\text{Cov}(\bm{\xi}^i,\bm{\xi}^j)$, $\bm{\xi}^i,\bm{\xi}^j\in\{\bm{\xi}^R,\bm{\xi}^p\}$ in (\ref{xi_form}). And the covariance of an observation $$\bm{\Omega}=\left[
\begin{array}{ccc}
	\bm{\Omega}^{R,R}& \bm{\Omega}^{R,p}\\
	\bm{\Omega}^{p,R}& \bm{\Omega}^{p,p}
\end{array}
\right], $$ $\bm{\Omega}^{i,j}=\text{Cov}(\textbf{v}^i,\textbf{v}^j),\ \textbf{v}^i,\textbf{v}^j\in\{\textbf{v}^R,\textbf{v}^p\}$. Then the augmented covariance $\textbf{P}_{\text{aug}}$ is computed as
\begin{equation}
	\textbf{P}_{\text{aug}}=\left[
	\begin{array}{cccccc}
		\textbf{P}^{R,R}& \textbf{P}^{R,R}\textbf{M}_1^{T}& \textbf{P}^{R,p}& \textbf{P}^{R,p}\textbf{M}_2^{T}\\
		\textbf{M}_1P^{R,R}& \textbf{P}_f^{R,R}& \textbf{M}_1\textbf{P}^{R,p}& \textbf{P}_f^{R,p}\\
		\textbf{P}^{p,R}& \textbf{P}^{p,R}\textbf{M}_1^{T}& \textbf{P}^{p,p}& \textbf{P}^{p,p}\textbf{M}_2^{T}\\
		\textbf{M}_2\textbf{P}^{p,R}& (\textbf{P}_f^{R,p})^T& \textbf{M}_2\textbf{P}^{p,p}& \textbf{P}_f^{p,p}\\
	\end{array}
	\right],
	\label{Pnew_app}
\end{equation}
where
\begin{equation}
	\begin{array}{rll}
		\textbf{M}_1&=[\textbf{I}_3 \ \textbf{0}_{3,3K}],\\
		\textbf{M}_2&=[\textbf{I}_3 \ \textbf{0}_{3,3K}],\\
		\textbf{P}_f^{R,R}&=\textbf{M}_1\textbf{P}^{R,R}\textbf{M}_1^{T}+\hat{\textbf{R}}^r\bm{\Omega}^{R,R}(\hat{\textbf{R}}^r)^T,\\
		\textbf{P}_f^{R,p}&=\textbf{M}_1\textbf{P}^{R,p}\textbf{M}_2^{T}+\hat{\textbf{R}}^r\bm{\Omega}^{R,p}(\hat{\textbf{R}}^r)^T,\\
		\textbf{P}_f^{p,p}&=\textbf{M}_2\textbf{P}^{p,p}\textbf{M}_2^{T}+\hat{\textbf{R}}^r\bm{\Omega}^{p,p}(\hat{\textbf{R}}^r)^T.
	\end{array}
\end{equation}

\subsection{Proof of Theorem 1}\label{Proof_Th1}
%
\begin{proof}
The observabilty matrix of the real system model (the ideal case: evaluated at the groundtruth) for RI-EKF is shown as
	\begin{equation}
		\begin{array}{rl}
			&\mathcal{\breve{\bm{O}}}^{RI}=\left[
			\begin{array}{cccccc}
				\breve{\textbf{H}}^{RI}_0\\
				\breve{\textbf{H}}^{RI}_1\breve{\textbf{F}}^{RI}_{0,0}\\
				\vdots\\
				\breve{\textbf{H}}^{RI}_{n+1}\breve{\textbf{F}}^{RI}_{n,0}
			\end{array}
			\right]\\
			&=\left[
			\begin{array}{cccccc}
				\breve{\textbf{H}}^{R,R^r}_0& \breve{\textbf{H}}^{R,R^f}_0& \textbf{0}_{3\times 3}& \textbf{0}_{3\times 3}\\
				\textbf{0}_{3\times 3}& \textbf{0}_{3\times 3}& \breve{\textbf{H}}^{p,p^r}_0& \breve{\textbf{H}}^{p,p^f}_0\\
				\vdots& \vdots& \vdots& \vdots\\
				\breve{\textbf{H}}^{R,R^r}_{n+1}& \breve{\textbf{H}}^{R,R^f}_{n+1}& \textbf{0}_{3\times 3}& \textbf{0}_{3\times 3}\\
				\textbf{0}_{3\times 3}& \textbf{0}_{3\times 3}&\breve{\textbf{H}}^{p,p^r}_{n+1}& \breve{\textbf{H}}^{p,p^f}_{n+1}\\
			\end{array}
			\right],
		\end{array}
	\end{equation}
	where $$\begin{aligned}
		&\breve{\textbf{H}}^{RI}_0=\left[
		\begin{array}{cccccc}
			-({\textbf{R}}_{0}^r)^T& ({\textbf{R}}_{0}^r)^T& \textbf{0}_{3\times 3}& \textbf{0}_{3\times 3}\\
			\textbf{0}_{3\times 3}& \textbf{0}_{3\times 3}& -({\textbf{R}}_{0}^r)^T& ({\textbf{R}}_{0}^r)^T
		\end{array}
		\right],\\
		&\breve{\textbf{H}}^{R,R^r}_{i+1}=\breve{\textbf{H}}^{p,p^r}_{i+1}=-({\textbf{R}}_{i+1}^r)^T,\ \forall i=0,\cdots, n,\\
		&\breve{\textbf{H}}^{R,R^f}_{i+1}=\breve{\textbf{H}}^{p,p^f}_{i+1}=({\textbf{R}}_{i+1}^r)^T,\ \forall i=0,\cdots, n,\\
	\end{aligned}$$
	$\breve{\textbf{H}}^{RI}_{i+1}$ is the Jacobian matrix for the $(i+1)$-th step observation model evaluated at the true state $\breve{\textbf{X}}^{RI}_{i+1}$, and $\breve{\textbf{F}}^{RI}_{i,0}=\breve{\textbf{F}}^{RI}_i\breve{\textbf{F}}^{RI}_{i-1}\cdots \breve{\textbf{F}}^{RI}_0$ is the identity matrix, $\breve{\textbf{F}}^{RI}_j$ is the Jacobian matrix for the $j$-th step propagation model evaluated at the true state $\breve{\textbf{X}}^{RI}_j$, $ j=0,\cdots, i$. 
	Based on the linearization method of RI-EKF introduced above, 
	we can obtain the unobservable subspace $\mathcal{\breve{N}}^{RI}$ by
	\begin{equation}
		\mathcal{\breve{N}}^{RI}=\mathop{\text{span}} _{col.}\left[
		\begin{array}{cccccc}
			\textbf{I}_{3}& \textbf{0}_{3,3}\\
			\textbf{I}_{3}& \textbf{0}_{3,3}\\
			\textbf{0}_{3,3}& \textbf{I}_{3}\\
			\textbf{0}_{3,3}& \textbf{I}_{3}
		\end{array}
		\right],
		\label{nullspace_app}
	\end{equation}
	whose dimension is 6.
	Then consider the observabilty matrix $\mathcal{\hat{\bm{O}}}^{RI}$ evaluated by RI-EKF estimations, 
	\begin{equation}
		\begin{array}{rl}
			&\mathcal{\hat{\bm{O}}}^{RI}=\left[
			\begin{array}{cccccc}
				\hat{\textbf{H}}^{RI}_0\\
				\hat{\textbf{H}}^{RI}_1\hat{\textbf{F}}^{RI}_{0,0}\\
				\vdots\\
				\hat{\textbf{H}}^{RI}_{n+1}\hat{\textbf{F}}^{RI}_{n,0}
			\end{array}
			\right]\\
			&=\left[
			\begin{array}{cccccc}
				\hat{\textbf{H}}^{R,R^r}_0& \hat{\textbf{H}}^{R,R^f}_0& \textbf{0}_{3\times 3}& \textbf{0}_{3\times 3}\\
				\textbf{0}_{3\times 3}& \textbf{0}_{3\times 3}& \hat{\textbf{H}}^{p,p^r}_0& \hat{\textbf{H}}^{p,p^f}_0\\
				\vdots& \vdots& \vdots& \vdots\\
				\hat{\textbf{H}}^{R,R^r}_{n+1}& \hat{\textbf{H}}^{R,R^f}_{n+1}& \textbf{0}_{3\times 3}& \textbf{0}_{3\times 3}\\
				\textbf{0}_{3\times 3}& \textbf{0}_{3\times 3}&\hat{\textbf{H}}^{p,p^r}_{n+1}& \hat{\textbf{H}}^{p,p^f}_{n+1}\\
			\end{array}
			\right],
		\end{array}
	\end{equation}
	where $$\begin{aligned}
		&\hat{\textbf{H}}^{RI}_0=\left[
		\begin{array}{cccccc}
			-(\hat{\textbf{R}}^r_{0|0})^T& (\hat{\textbf{R}}^r_{0|0})^T& \textbf{0}_{3\times 3}& \textbf{0}_{3\times 3}\\
			\textbf{0}_{3\times 3}& \textbf{0}_{3\times 3}& -(\hat{\textbf{R}}^r_{0|0})^T& (\hat{\textbf{R}}^r_{0|0})^T\\
		\end{array}\right],\\
		&\hat{\textbf{H}}^{R,R^r}_{i+1}=\hat{\textbf{H}}^{p,p^r}_{i+1}=-(\hat{\textbf{R}}_{i+1|i}^r)^T,\ \forall i=0,\cdots, n,\\
		&\hat{\textbf{H}}^{R,R^f}_{i+1}=\hat{\textbf{H}}^{p,p^f}_{i+1}=(\hat{\textbf{R}}_{i+1|i}^r)^T,\ \forall i=0,\cdots, n,\\
	\end{aligned}$$
	 $\hat{\textbf{H}}^{RI}_{i+1}$ is the Jacobian matrix for the $(i+1)$-th step observation model evaluated at the estimated state $\hat{\textbf{X}}^{RI}_{i+1|i}$, and $\hat{\textbf{F}}^{RI}_{i,0}=\hat{\textbf{F}}^{RI}_i\hat{\textbf{F}}^{RI}_{i-1}\cdots \hat{\textbf{F}}^{RI}_0$, $\hat{\textbf{F}}_j$  is the Jacobian matrix for the $j$-th step propagation model evaluated at the estimated state $\hat{\textbf{X}}^{RI}_{j|j}$, $j=0,\cdots i$. The unobservable subspace $\mathcal{\hat{N}}^{RI}$ of $\mathcal{\hat{\bm{O}}}^{RI}$  is the same as (\ref{nullspace_app}). 
\end{proof}

\subsection{Proof of Theorem 2}\label{Proof_Th2}
\begin{proof}
	The observabilty matrix of the stdandard EKF based on state estimates is constructed as
	\begin{equation}
		\mathcal{\hat{\bm{O}}}^{Std}=\left[
		\begin{array}{cccccc}
			\hat{\textbf{H}}^{Std}_0\\
			\hat{\textbf{H}}^{Std}_{1}\hat{\textbf{F}}^{Std}_{0,0}\\
			\vdots\\
			\hat{\textbf{H}}^{Std}_{n+1}\hat{\textbf{F}}^{Std}_{n,0}
		\end{array}
		\right]
	\end{equation}
	where $\hat{\textbf{H}}^{Std}_{i+1}$ is the Jacobian matrix evaluated at the prediction $\textbf{X}_{i+1|i}$ for the $(i+1)$-th step observation model in (\ref{FGH_std}), and
	\begin{equation}
		\begin{array}{rl}
			&\hat{\textbf{F}}^{Std}_{i,0}=\hat{\textbf{F}}^{Std}_i\hat{\textbf{F}}^{Std}_{i-1}\cdots \hat{\textbf{F}}^{Std}_0\\
			&=\left[
			\begin{array}{cccccc}
				\textbf{I}_6&\textbf{0}_{6\times 6}\\
				\left[\begin{array}{cccccc}
					\sum_{j=0}^i(\textbf{p}^r_{j|j}-\textbf{p}^r_{j+1|j})^{\land}&\textbf{0}_{3\times 3}\\
					\textbf{0}_{3\times 3}&\textbf{0}_{3\times 3}
				\end{array}
				\right]&\textbf{I}_6
			\end{array}
			\right]
		\end{array}
	\end{equation}
	$\hat{\textbf{F}}^{Std}_{j}$ is Jacobian evaluated at estimated state $\textbf{X}_{j|j}$ of process model for standard EKF in (\ref{FGH_std}). 
	
	In ideal case, the Jacobians are evaluated at the true state $\breve{\textbf{X}}_i$. Then, we have $$\textbf{p}^r_{i+1|i+1}= \textbf{p}^r_{i+1|i}$$ and $$\textbf{p}_{i+1|i}^f=\textbf{p}^f,\ \forall i.$$ And thus, $\hat{\textbf{H}}^{Std}_{i}\hat{\textbf{F}}^{Std}_{i-1,0}$ in $\mathcal{\breve{\bm{O}}}^{Std}$ for ideal case becomes  
	\begin{equation}
		\begin{array}{l}
			\hat{\textbf{H}}^{Std}_{i}\hat{\textbf{F}}^{Std}_{i-1,0}=\\
			\left[
			\begin{array}{cccccc}
				-({\textbf{R}}^r_{i})^T& ({\textbf{R}}^r_{i})^T& \textbf{0}_{3\times 3}& \textbf{0}_{3\times 3}\\
				-({\textbf{R}}^r_{i})^T(\textbf{p}^f-\textbf{p}^r_0)^{\land}& \textbf{0}_{3\times 3}& -({\textbf{R}}^r_{i})^T& ({\textbf{R}}^r_{i})^T\\
			\end{array}
			\right].
		\end{array}
	\end{equation}
	Then unobservable subspace based on the ground truth is obtained as 
	\begin{equation}
		\mathcal{\breve{N}}^{Std}=\mathop{\text{span}} _{col.}\left[
		\begin{array}{cccccc}
			\textbf{0}_{3,3}& \textbf{I}_{3}\\
			\textbf{0}_{3,3}& \textbf{I}_{3}\\
			\textbf{I}_{3}& (\textbf{p}^r_0)^{\land}\\
			\textbf{I}_{3}& (\textbf{p}^f)^{\land}
		\end{array}
		\right].
	\end{equation}
	Therefore, the dimension of unobservable subspace of system is 6. The same conclusion is for the case of multiple features.
	
	However, in practice, generally
	$$\begin{array}{ll}
		\textbf{p}^r_{i+1|i+1}&\neq \textbf{p}^r_{i+1|i},\\
		\textbf{p}^f_{i+1|i}&\neq \textbf{p}^f_{j+1|j},\ \forall i\neq j,
	\end{array} $$ and thus the unobservable subspace of standard EKF evaluated by the estimates becomes
	\begin{equation}
		\mathcal{\hat{N}}^{Std}=\mathop{\text{span}} _{col.}[\textbf{0}_{3\times 3},\textbf{0}_{3\times 3},\textbf{I}_3,\textbf{I}_3]^T,
	\end{equation}
	whose dimension is 3. 
\end{proof}


\end{document}